\documentclass[journal,twoside]{IEEEtran}
\usepackage[cmex10]{amsmath}
\usepackage{amssymb}
\usepackage{amsfonts}
\usepackage{amsthm}
\usepackage{algorithm}
\usepackage{algpseudocode}
\usepackage{graphicx}
\usepackage{epsfig}
\usepackage{cite}
\usepackage{tensor}
\usepackage[caption=false,font=footnotesize]{subfig}
\usepackage{color}
\usepackage[dvipsnames]{xcolor}
\usepackage[pdfa]{hyperref}
\hypersetup{
    colorlinks=true,
    linkcolor=Blue,
    filecolor=magenta,      
    urlcolor=Green,
    citecolor=purple,
}

\graphicspath{{figures/}}
\DeclareGraphicsExtensions{.pdf}
\theoremstyle{definition}
\newtheorem{remark}{Remark}
\newtheorem{assumption}{Assumptions}

\newtheorem*{problem}{Problem statement}

\newcommand\T{{\hspace{-0pt}\intercal}}
\newcommand*\diff{\mathop{}\!\mathrm{d}}

\newcommand{\comments}[1]{\textcolor{blue}{#1}}

\newtheorem{prop}{Proposition}
\makeatletter
\DeclareFontFamily{U}{tipa}{}
\DeclareFontShape{U}{tipa}{m}{n}{<->tipa10}{}
\newcommand{\arc@char}{{\usefont{U}{tipa}{m}{n}\symbol{62}}}%

\newcommand{\arc}[1]{\mathpalette\arc@arc{#1}}

\newcommand{\arc@arc}[2]{%
  \sbox0{$\m@th#1#2$}%
  \vbox{
    \hbox{\resizebox{\wd0}{\height}{\arc@char}}
    \nointerlineskip
    \box0
  }%
}
\makeatother

\usepackage{siunitx}
\DeclareSIUnit\mt{\milli\tesla} %% A method for say short cut or new unit!
\renewcommand{\baselinestretch}{0.97}

\begin{document}

\title{On Radiation-Based Thermal Servoing:\\ New Models, Controls and Experiments}

\author{Luyin Hu, David Navarro-Alarcon, Andrea Cherubini and Mengying Li%
\thanks{This work is supported in part by the Research Grants Council (RGC) of Hong Kong under grant number 14203917, in part by PROCORE-France/Hong Kong Joint Research Scheme sponsored by the RGC and the Consulate General of France in Hong Kong under grant F-PolyU503/18, in part by the Jiangsu Industrial Technology Research Institute Collaborative Funding Scheme under grant 43-ZG9V, in part by the Key-Area Research and Development Program of Guangdong Province 2020 under project 76 and in part by the Hong Kong PolyU under grants ZZHJ and YBYT.}%
\thanks{Luyin Hu, David Navarro-Alarcon and Mengying Li are with The Hong Kong Polytechnic University, Kowloon, Hong Kong. Corresponding author email: {\texttt{\small dna@ieee.org}}}%
\thanks{Andrea Cherubini is with LIRMM, University of Montpellier CNRS, Montpellier, France.}
}

\bstctlcite{IEEEexample:BSTcontrol}

\markboth{}%
{Hu \MakeLowercase{\textit{et al.}}: On Radiation-Based Thermal Servoing}
\maketitle

\begin{abstract}
In this paper, we introduce a new sensor-based control method that regulates (by means of robot motions) the heat transfer between a radiative source and an object of interest.
This valuable sensorimotor capability is needed in many industrial, dermatology and field robot applications, and it is an essential component for creating machines with advanced thermo-motor intelligence. 
To this end, we derive a geometric-thermal-motor model which describes the relation between the robot's active configuration and the produced dynamic thermal response. 
We then use the model to guide the design of two new thermal servoing controllers (one model-based and one adaptive), and analyze their stability with Lyapunov theory.
To validate our method, we report a detailed experimental study with a robotic manipulator conducting autonomous thermal servoing tasks. To the best of the authors' knowledge, this is the first time that temperature regulation has been formulated as a motion control problem for robots.
\end{abstract}

\begin{IEEEkeywords}
				Thermoception, visual servoing, sensor-based control, robotic manipulation
\end{IEEEkeywords}
\IEEEpeerreviewmaketitle
\begin{figure*}
				\centering
				\includegraphics[width =2 \columnwidth]{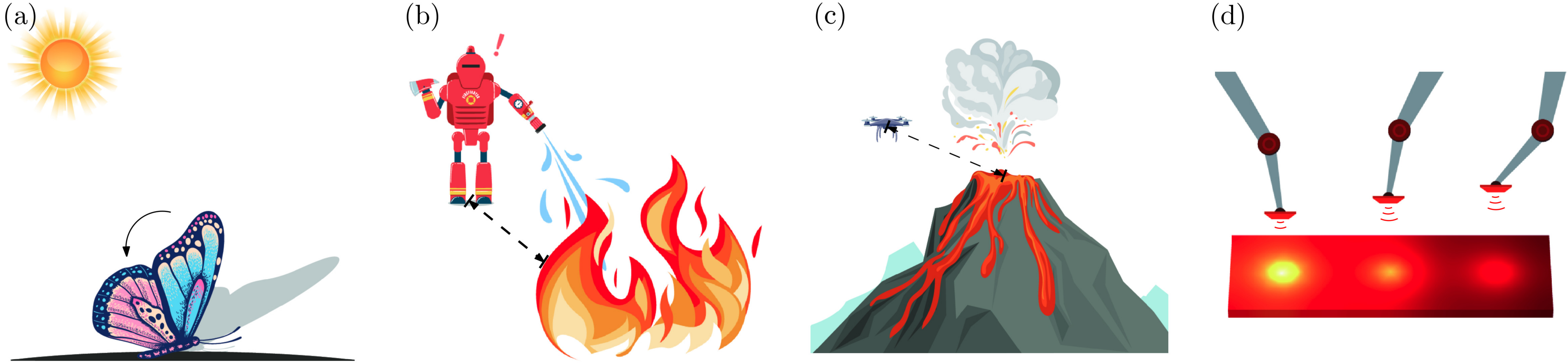}
				\caption{Creatures and robots with thermo-motor intelligence: When exposed to the sun, butterflies adjust their wings configuration to control their temperature (a); robotic systems with thermal servoing algorithms can be used for firefighting (b), volcano exploration (c) and industrial applications (d).}
				\label{fig:illustrations}
\end{figure*}
\section{INTRODUCTION}
\IEEEPARstart{T}{hermal servoing} is a feedback control problem that deals with the regulation of an object's  temperature by means of motor actions of a rigid robot, which can either manipulate the object or the heat source.
It is a frontier problem that has numerous important applications (e.g. in industrial process control, cosmetic dermatology, fire-fighting missions, etc.) where temperature needs to be dynamically controlled and the environment is uncertain.
The quality, performance and safety of these (otherwise open-loop) applications can be considerably improved by incorporating thermal sensorimotor capabilities.

From a control systems perspective, the automation of this type of temperature-critical tasks requires: (a) the computation of a geometric-thermal-motor (GTM) model describing the relation between the robot's motion and the consequent thermal response, and (b) the development of a sensor-based strategy to autonomously impose a desired heat profile onto the surface of interest.
Note that unlike other perception modalities for robot control (e.g. vision \cite{dna2014_ijrr}, proximity \cite{Journals:Cherubini2013}, touch \cite{marie2020_ral}, audition \cite{Proceedings:Magassouba2016} and even smell \cite{Proceedings:Rahbar2017}), thermoception has not been fully formalized in the literature as a bona fide feedback signal for motion control. In the robotics community, we still lack the framework to fully exploit it.
Up to now, the overwhelming use of thermoception in robotics has been to monitor processes but not to establish \emph{explicit thermal servo-loops} \cite{Journals:Gade2014}, which are needed to accurately regulate temperature.
Our aim in this paper is precisely to develop the necessary framework that enables the design of thermal servoing controls with radiative heat sources.

\subsection{Related Work}
Although thermal sensing is a mature technology and has a rich history in the automation of many tasks (see e.g. \cite{Journals:Benli2019,Journals:Cao2018,Jorunals:He2015,Journals:Lai2015}), its use as a feedback signal for robot control has not been sufficiently studied in the literature \cite{cherubini2020_frontneuro}, where only a few works have addressed this challenging servo-control problem.
Some representative works that deal with explicit thermal control include: \cite{fuel_cell}, where a fuzzy controller is developed to regulate the temperature of a fuel cell actuator; \cite{shape_memory_alloy}, where the influence of temperature in the deformation behavior of a surgical robot is investigated, and an explicit thermal regulator is designed; \cite{welding}, where a control method is designed to maintain a constant tool temperature by adjusting the spindle speed in a stir friction welding robot.
However, in these types of methods, temperature control is achieved by directly modulating the power of the heat-generating components. This approach is not suitable when considering external heat sources, e.g. wildfires \cite{firefighting} and sunlight \cite{solar_tracking}, or when the source's power should not be varied, e.g. in cosmetic procedures \cite{muddassir2020_tmech}.

A different strategy is to use sensor-based control, i.e. to dynamically change the source-object geometric configuration to achieve a desired thermal response (similar to what many organisms do \cite{Journals:McKemy2007}). This can be easily done by rigid robots, since their basic function is motion control.
Such approach demands the development of appropriate models that can effectively capture the system's GTM relations.
This idea has been partially demonstrated in \cite{car_mold}, where the optimal fixed location of multiple radiating heaters in a process is automatically calculated to evenly imprint a desired thermal profile onto a surface. Yet, the heater is static and the method requires exact knowledge of all thermodynamic parameters (which are generally unknown).
The proposed approach has also the potential to be used e.g. in fire-fighting \cite{Proceedings:Imdoukh2017} or volcano exploration robots \cite{volcano} to calculate optimal trajectories that avoid overheating or damaging the robot's components.

The dynamic coupling between temperature and motion may seem unintuitive for humans \cite{Journals:Fu2016}. 
Yet, many organisms extensively exploit these relations. 
For instance, marine animals that inhabit hydro-thermal vents manage their energetic demands by controlling their proximity to the source \cite{Journals:Lee2003}. 
Butterflies increase their body temperature by basking dorsally/laterally to maximize sun expose \cite{butterflies}. Ground-dwelling insects adaptively change their body configuration with respect to the sun-heated ground to regulate their temperature \cite{insect_thermoregulationt}.
Such advanced thermoception-based behaviors can be used to solve many real-world problems (see Fig. \ref{fig:illustrations}). 
However, these capabilities have not yet been fully incorporated in robot control, a discipline with good track record of borrowing inspiration from nature \cite{Journals:Natale2002,Journals:Arechavaleta2008,Journals:Na2020}, but which seems to be lagging in this direction.

\subsection{Our Contribution}
As a feasible solution to the above-mentioned issues, in this paper, we present a rigorous formulation for robot thermal servoing with radiative sources. The main contributions are summarized as follows:
\begin{itemize}
    \item We develop an efficient algorithm for computing in real-time the radiation-based thermal interaction matrix which relates robot velocity and object temperature rate.
    \item We present a novel robot control method for automatically regulating the temperature of grasped objects.
    \item We report a detailed experimental study to validate the proposed theory.
\end{itemize}
To the best of the authors' knowledge, this is the first time that temperature regulation has been formulated in the literature as a robot servoing problem.
The proposed approach has the potential to advance the development of multimodal motion controllers for robots.

The rest of the paper is as follows: Section \ref{section:mathematical modelling} presents the mathematical models; Section \ref{section:controller design} derives the controller; Section \ref{section:results} reports results; Section \ref{conclusion} gives final conclusions.

\section{MATHEMATICAL MODELING}
\label{section:mathematical modelling}

\subsection{Notation}
Throughout this manuscript, we denote all \emph{column} vectors by small bold letters, e.g. $\mathbf v \in \mathbb{R}^{n\times 1}$, and matrices by capital bold letters, e.g. $\mathbf{M} \in \mathbb{R}^{m \times n}$.

\subsection{Heat Transfer Model}
\label{section:GTM Overview}
In the following sections, we introduce basic thermodynamic concepts (we refer the reader to \cite{Fundamentals_of_heat_and_mass_transfer,modest2013radiative}) that are needed for developing the system's GTM model. 
To this end, consider a robot manipulator with end-effector configuration denoted by a vector $\mathbf x \in \mathbb{R}^n$. 
The robot rigidly grasps (through an adiabatic layer) a planar object whose surface temperature is to be controlled by changing the relative pose to a heat source.
The heat transfer model is composed of three main parts (depicted in Fig. \ref{fig:model}): (i) heat source, (ii) heat collector (i.e. the object), and (iii) surrounding environment. 
Thermophysical parameters of different parts are denoted by the same symbol but with different subscripts. 
We denote the (constant) temperature of the heat source and the (varying) temperature of the object by ${T}_{1}$ and $T_{2}$, respectively. 
We assume both temperatures to be spatially uniform during the heat transfer process. 
The environment temperature (assumed to be constant) is denoted by ${T}_3$. 
\begin{remark}
In this paper, we use the subscripts $i=1,2,3$ to denote the thermophysical parameters of the heat source, the object and the environment, respectively (a convention followed by many works dealing with heat transfer).
\end{remark}

Heat transfer occurs amongst the three parts whenever ${T}_{1}$, ${T}_{2}, {T}_{3}$ have different values. 
The direction of heat transfer is always from a high temperature part to a low temperature part. 
We denote the net energy transfer rate to the object by $Q_{2}$, where a positive value indicates energy inflow. 
We introduce \(q_{2}=Q_{2}/ A_2\) to represent the surface's net heat flux and $v=\diff T_2/\diff t$ to describe the temporal change of the measured temperature ${T}_{2}$. 
According to the energy conservation laws, these quantities satisfy the relation:
\begin{equation}
    v = \frac{1}{m_{2}c_{2}}Q_{2}
\label{eq:temp speed}
\end{equation}
where $m_{2}$ denotes object's mass and $c_{2}$ denotes the material's specific heat. To synthesize a thermal servoing controller, it is useful to find an expression of the following form:
\begin{equation}
v=f( \mathbf{x},{T}_{2})
\label{eq:thermal-geometric-model}
\end{equation}
which describes the thermal-geometric relation between the robot configuration and the temperature rate.

\subsection{Radiation Exchange Between Planar Surfaces}
\label{section:Radiation Exchange Between Surfaces}
In this subsection, we show how to calculate $Q_{2}$ between planar surfaces when thermophysical properties are known.  
According to different mechanisms involved in the heat transfer processes \cite{lienhard2005heat}, the object's net heat flux $q_{2}$ satisfies the expression $q_{2}=q_{rad}+q_{conv}+{q}_{cond}$, for radiative ${q}_{rad}$, convective ${q}_{conv}$ and conductive ${q}_{cond}$ fluxes. 
In our case of study, thermal radiation is the dominant heat transfer mode; Note that $q_{cond}$ and $q_{conv}$ are negligible since the object is grasped through an adiabatic layer and the source's temperature is much higher than those of the object and environment.
\begin{assumption}\label{assumptions_1}
We assume that the following conditions are satisfied during the task (see Fig. \ref{fig:model}):
\begin{enumerate}
    \item All surfaces have uniform temperature properties.
    \item All surfaces are gray, i.e. they are diffuse emitters with equal emittance and absorptance.
    \item The environment/room is modeled as a black body (i.e. $\varepsilon_{3}={\alpha}_{3}=1$, see the variables' definition below).
\end{enumerate}
\end{assumption}
The net energy transfer rate ${Q_{2}}$ has the following form:
\begin{equation}
Q_{2}=A_2 {q}_{rad}={A}_{2}({\alpha}_{2} {G}_{2}- {E}_{2})
\label{eq:radiative rate}
\end{equation}
where $A_{2}$ denotes the object's surface area, $G_{2}$ the radiative flux incident at the surface, $\alpha_{2} \in [0,1]$ the object's absorptance and $E_{2}$ the heat flux emitted by a surface (i.e. emissive power) which is approximated using the Stefan-Boltsman law \cite{Fundamentals_of_heat_and_mass_transfer}
\begin{equation}
E_{2}={\varepsilon}_{2} \sigma {T}_{2}^{4}
\label{emmisive heat}
\end{equation}
with $\varepsilon_{2} \in [0,1]$ the material's emittance, and $\sigma$ is the Stefan-Boltzmann constant.

\begin{figure}
				\centering
				\includegraphics[width = 0.9 \columnwidth]{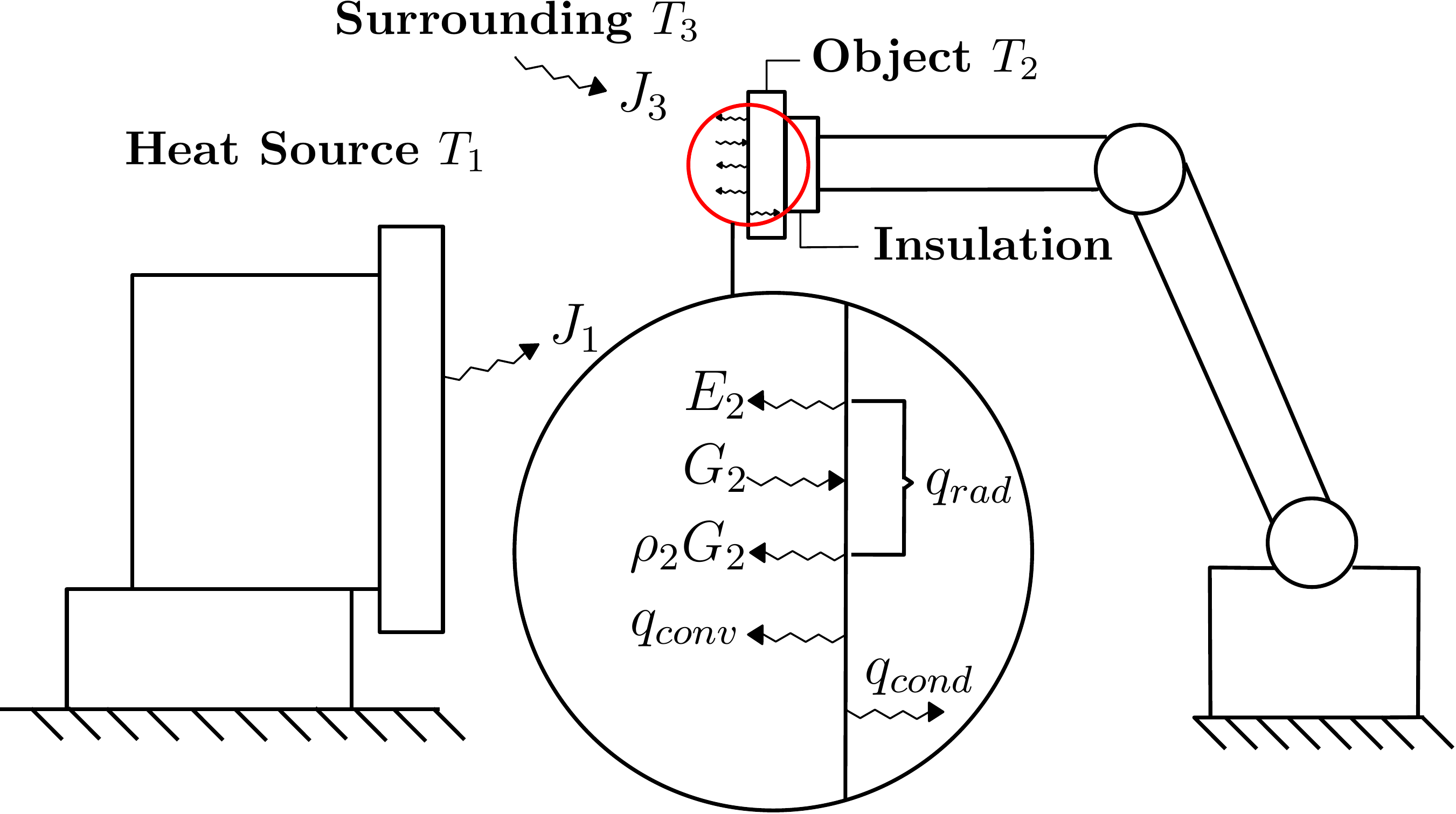}
				\caption{Representation of the heat transfer model. A part of the object surface is magnified to show the various heat transfer processes.}
				\label{fig:model}
\end{figure}

Note that for an opaque surface (i.e. with zero transmittance), its reflectance $\rho_{i}$ and absorptance $\alpha_{i}$ satisfy $\rho_{i}+{\alpha}_{i}=1$.  
Radiosity is defined as $J_i = {E}_{i}+{\rho}_{i} {G}_{i}$, and since for our heat source $E_1\gg\rho_{1}G_{1}$, we can fairly approximate it as $J_{1} \approx E_{1}$.

\remark{The view factor ${F}_{ij}$ represents the fraction of $J_i$ that is incident on surface $j$. 
The view factor depends on the end-effector configuration, i.e. $F_{ij}=F_{ij}(\mathbf x)$. Thus, its calculation is essential for deriving the geometric-thermal-motor model.
The detailed derivation of $F_{ij}$ for different cases and configurations is presented in later sections. 
Here, we assume $F_{ij}$ in known and only focus on the derivation of $Q_{2}$.}

The radiation incident to a surface is the summation of the corresponding portion of radiation coming from other surfaces. 
Thus, $G_i$ can be calculated from the expression:
\begin{equation}
A_{2} G_{2}=\sum_{j=1}^{3} F_{j2} A_{j} J_{j}=F_{12}A_{1}J_{1}+F_{22}A_{2}J_{2}+F_{32}A_{3}J_{3}
\label{eq:irradiation summation}
\end{equation}
where by using the reciprocity relation $A_{i} F_{i j}=A_{j} F_{j i}$, the summation rule $\sum_{j=1}^{N} F_{ij}=1$, and the view factor property for planar surfaces $F_{ii}=0$ (readers are referred to \cite{Fundamentals_of_heat_and_mass_transfer}), we can simplify \eqref{eq:irradiation summation} into:
\begin{equation}
\begin{split}
{A}_{2}{G}_{2}&=F_{21}{A}_{2}J_{1}+{F}_{23}A_{2}J_{3}\\
&=F_{21}{A}_{2}E_{1}+(1-F_{21})A_{2}E_3
\end{split}
\label{eq:A2G2}
\end{equation}
Substitution of \eqref{eq:A2G2} into \eqref{eq:radiative rate} yields:
\begin{equation}
Q_{2}=A_{2}\alpha_{2}(E_{1}-E_3)F_{21}+A_{2}\alpha_{2}E_{3}-A_{2}E_{2}
\label{eq:q2_derive}
\end{equation}
which we substitute alongside \eqref{emmisive heat} into \eqref{eq:temp speed} to obtain the following key expression for the object's temperature rate:
\begin{equation}
    \label{eq:temperature rate}
    v =\lambda_{1}F_{21} -\lambda_{2}{T}_{2}^4+\lambda_{3}
\end{equation}
for \emph{constant} scalar parameters $\lambda_{1}$, $\lambda_{2}$, and $\lambda_{3}$ satisfying
\begin{equation}
     \lambda_{1}=\frac{A_2{\alpha_{2}\sigma}({\varepsilon_{1}}{T}_{1}^4-{T}_{3}^4 )}{m_{2}c_{2}},~
     \lambda_{2}=\frac{A_2{\varepsilon_{2}}\sigma}{m_{2}c_{2}},~
     \lambda_{3}=\frac{A_2{\alpha}_{2}\sigma {T}_{3}^4}{m_{2}c_{2}}
     \label{eq:lambda values}
\end{equation}

\subsection{View Factor Analytical Definition}
\label{section:view factor}
In this section, we provide the general expression of $F_{2 1}$, which we will instantiate (in the following sections) for various configurations.
To this end, consider the elementary areas ${\diff}A_{1}$ and ${\diff}A_{2}$ on the source and object surfaces, respectively. 
These areas are separated by a length $r$ that forms polar angles $\theta_1$ and $\theta_2$ (see Fig. \ref{fig:VF}). 
The classical definition of the view factor is:
\begin{equation}
F_{2 1}=\frac{1}{A_{2}} \int_{A_{2}} \int_{A_{1}} \frac{\cos \theta_{2} \cos \theta_{1}}{\pi r^{2}} \diff A_{2} \diff A_{1}
\label{eq:VF surface integral}
\end{equation}{}
The solution of \eqref{eq:VF surface integral} is usually complicated to derive. 
A variety of methods \cite{viewfactor1,viewfactor2,viewfactor3,vujivcic2006numerical} have been proposed to calculate it. 
Here, we use the method in \cite{sparrow2018radiation}, which converts the double surface integrals into double contour integrals as follows:
\begin{equation}
\label{eq:VF contour integral}
F_{21}=\frac{1}{2 \pi A_{2}} \oint_{\Gamma_{1}} \oint_{\Gamma_{2}} \ln s \diff \mathbf s_{2} \cdot \diff \mathbf s_{1},
\end{equation}
where $\Gamma_i$ denotes the contour of the $i$th surface, $\mathbf s_i$ the position vector of an arbitrary point on boundary ${\Gamma}_i$, and $s=\| \mathbf{s}_{2}-\mathbf{s}_{1} \|$ the distance between two contour points. 
The advantage of using this approach is its efficient computation time \cite{rao1996efficient}.

\subsection{Thermal Servoing with Parallel Circular Surfaces}
\label{section: parallel surfaces}
\begin{figure}
				\centering
				\includegraphics[width =0.8\columnwidth]{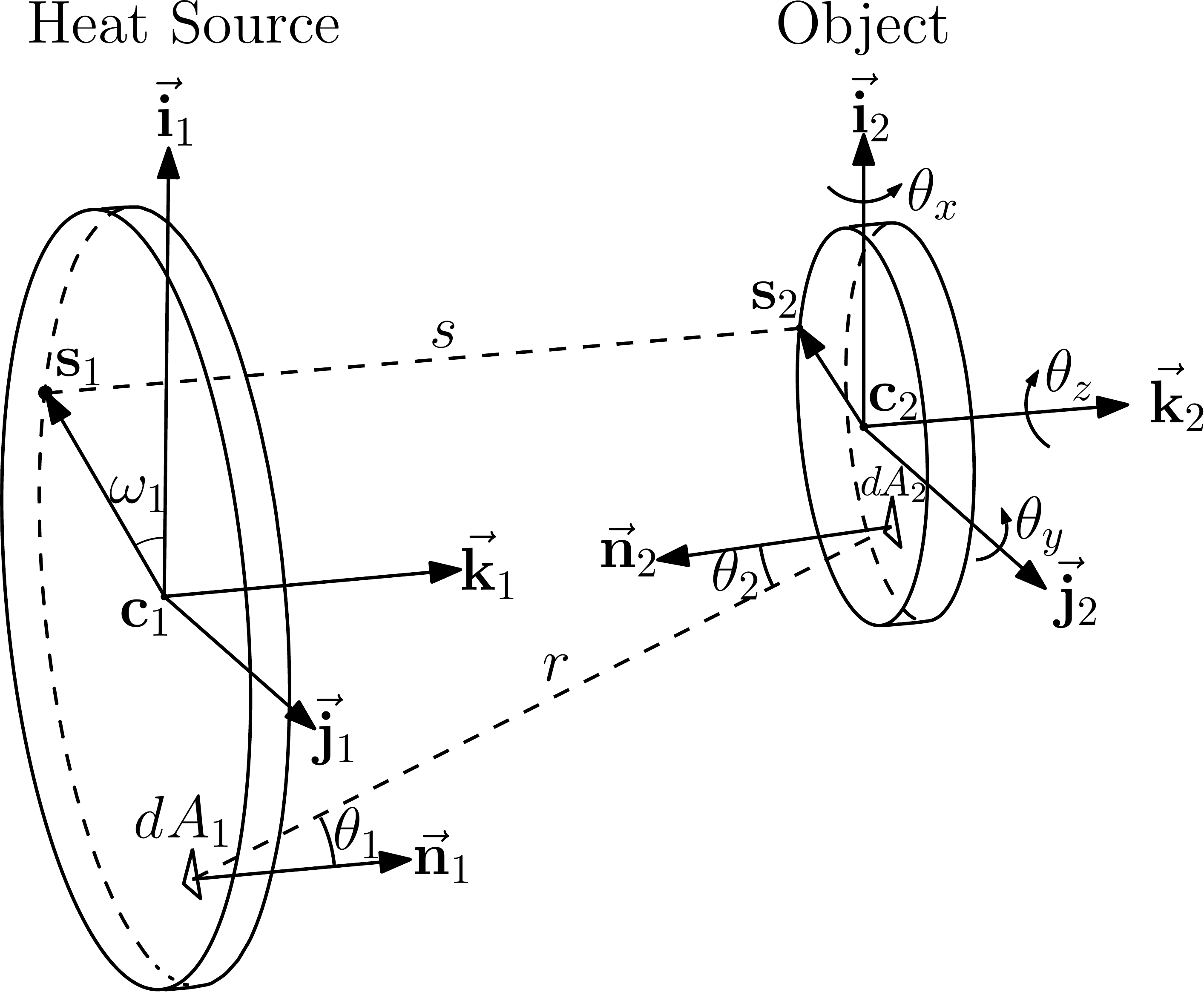}
				\caption{Geometry of the view factor between two elementary surfaces.}
				\label{fig:VF}
\end{figure}
In this section, we derive the thermal servoing model for two parallel source-object surfaces. To this end, we denote the surfaces' center and radius by $\mathbf{c}_i$ and $r_i$, respectively. 
The origin of the coordinate system $\vec{\mathbf{i}}_1\vec{\mathbf{j}}_1\vec{\mathbf{k}}_1$ is set at $\mathbf{c}_1$, with a unit basis vector $\vec{\mathbf{k}}_1$ along the normal $\vec{\mathbf{n}}_{1}$, and a unit basis vector $\vec{\mathbf{i}}_{1}$ perpendicular to the ground. 
We define $\vec{\mathbf{i}}_2\vec{\mathbf{j}}_2\vec{\mathbf{k}}_2$ as the translation of $\vec{\mathbf{i}}_1\vec{\mathbf{j}}_1\vec{\mathbf{k}}_1$, with origin at $\mathbf{c}_2$. 
The scalars $\omega_i$ denote the angle between ${\vec{\mathbf i}}_i$ and $\mathbf{s}_i$.
We set the frames' centers at $\mathbf{c}_1=[0,0,0]^\T $ and $\mathbf{c}_2= [p_{1},p_{2},p_{3}]^\T$, with respect to $\vec{\mathbf{i}}_1\vec{\mathbf{j}}_1\vec{\mathbf{k}}_1$.
The parametric position vectors $\mathbf{s}_i$ are then computed as:
\begin{equation}
\begin{split}
\mathbf{s}_1&= \begin{bmatrix}
{r}_{1}\cos{\omega}_{1}&{r}_{1}\sin{\omega}_{1}&0 \end{bmatrix}^\T
\\
\mathbf{s}_2 &=\begin{bmatrix}
{r}_{2}\cos{\omega}_{2}+p_1&{r}_{2}\sin{\omega}_{2}+p_{2}&p_3 \end{bmatrix} ^\T.
\label{eq:parametric equation2}
\end{split}
\end{equation}
Their differential changes satisfy the following relations:
\begin{equation}
\begin{split}
&\diff{\mathbf{s}_{1}} = \begin{bmatrix}
-{r}_{1}\sin{\omega}_{1}\diff{{\omega}_{1}}&{r}_{1}\cos{\omega}_{1}\diff{{\omega}_{1}}&0 \end{bmatrix}^\T\\
&\diff{{\mathbf{s}_{2}}} =\begin{bmatrix}
-{r}_{2}\sin{\omega}_{2}\diff{{\omega}_{2}}&{r}_{2}\cos{\omega}_{2}\diff{{\omega}_{2}}&0 \end{bmatrix}^\T \\
&\diff\mathbf{s}_{1}\cdot\diff\mathbf{s}_{2}=r_{1}r_{2}\cos(\omega_{1}-\omega_{2})\diff\omega_{1}\diff\omega_{2}.
\label{eq:ds1ds2}
\end{split}
\end{equation}
Then, the distance $s = \| \mathbf{s}_{2}-\mathbf{s}_{1} \|$ can be derived as:
\begin{multline}
s=s(p_{1}, p_{2}, p_{3}, \omega_{1}, \omega_{2})= \left({p_1}^2+{p_2}^2+{p_3}^2+ \right.\\
2{p}_{1}(r_{2}\cos{\omega_2}-{r}_{1}\cos{\omega}_{1})
+2{p}_{2}(r_{2}\sin{\omega_2}-{r}_{1}\sin{\omega}_{1})\\
\left.+{r_1}^2+{r_2}^2-2{r_1}{r_2}\cos(\omega_2-\omega_1) \right)^\frac{1}{2}.
\label{eq:s}
\end{multline}
By substituting \eqref{eq:ds1ds2}--\eqref{eq:s} into \eqref{eq:VF contour integral}, $F_{21}$ can be calculated with the expression:
\begin{equation}
\label{eq:instantaneous view factor}
F_{21}=\int_{0}^{2 \pi} \int_{2 \pi}^{0} \frac{ {r}_{1} {r}_{2} }{2 \pi \mathrm{A}_{2}} \cos \left({\omega}_{1}-{\omega}_{2}\right)
\ln {s(\mathbf x, \omega_{1}, \omega_{2})}\diff{\omega}_{2} \diff{\omega}_{1}  
\end{equation}
where we define the end-effector position as $\mathbf{x} =  [p_{1} , p_{2} , p_{3}{]}^\T$.
By injecting \eqref{eq:instantaneous view factor} into \eqref{eq:temperature rate}, we can finally obtain the system's thermal-geometric relation:
\begin{equation}
    \label{eq:final temperature rate}
    v=f(\mathbf{x},T_{2})=\lambda_{1}{F}_{21} -\lambda_{2}{T_{2}}^4+\lambda_3
\end{equation}
where $f(\cdot)$ is the function in \eqref{eq:thermal-geometric-model}.
By differentiating \eqref{eq:final temperature rate}, we obtain the following key dynamic model:
\begin{equation}
    \dot{v}=\mathbf{l} \cdot {\mathbf{u}}-4\lambda_{2}{T_{2}}^3v
    \label{eq:dynamic system}
\end{equation}
where $\mathbf{l}=\lambda_{1}\frac{\partial F_{21}}{\partial \mathbf x}^\T$ denotes the interaction/Jacobian matrix (vector, in this example), and $\mathbf {u} = \dot{\mathbf x}\in\mathbb R^n$ is the robot's Cartesian velocity. 
The above expression is used for designing control laws for $\dot{\mathbf x}$ in the following sections. 
By using Leibniz integral rule\cite{flanders1973differentiation}, the interaction matrix can be expressed as:
\begin{equation}
\label{eq:3dof interaction matrix}
\mathbf{l}=
\begin{bmatrix}
\displaystyle \int_{0}^{2 \pi} \int_{2 \pi}^{0}  h(p_{1}+r_{2}\cos\omega_{2}-r_{1}\cos\omega_{1})\diff \omega_{2} \diff \omega_{1}\\
\displaystyle\int_{0}^{2 \pi} \int_{2 \pi}^{0}
h(p_{2}+r_{2}\sin\omega_{2}-r_{1}\sin\omega_{1}) \diff \omega_{2} \diff \omega_{1}\\
\displaystyle\int_{0}^{2 \pi} \int_{2 \pi}^{0}  h {p}_{3} \diff \omega_{2} \diff \omega_{1}
\end{bmatrix}
\end{equation}
with the scalar $h$ defined as:
\begin{equation}
h=\lambda_{1}\frac{ r_{1}r_{2}\cos(\omega_{1}-\omega_{2})}{2\pi{A}_{2}s^2}\
\end{equation}
Note that since it is hard to analytically compute the above double integrals, we use a numerical method \cite{2020SciPy-NMeth} to approximate $\mathbf{l}$ in real-time. 
See Appendix A for details.

\subsection{Circular Surfaces in Arbitrary Configurations}

\label{section:Circular Surfaces at Arbitrary Configurations}
\begin{figure}
				\centering
				\includegraphics[width =0.8 \columnwidth]{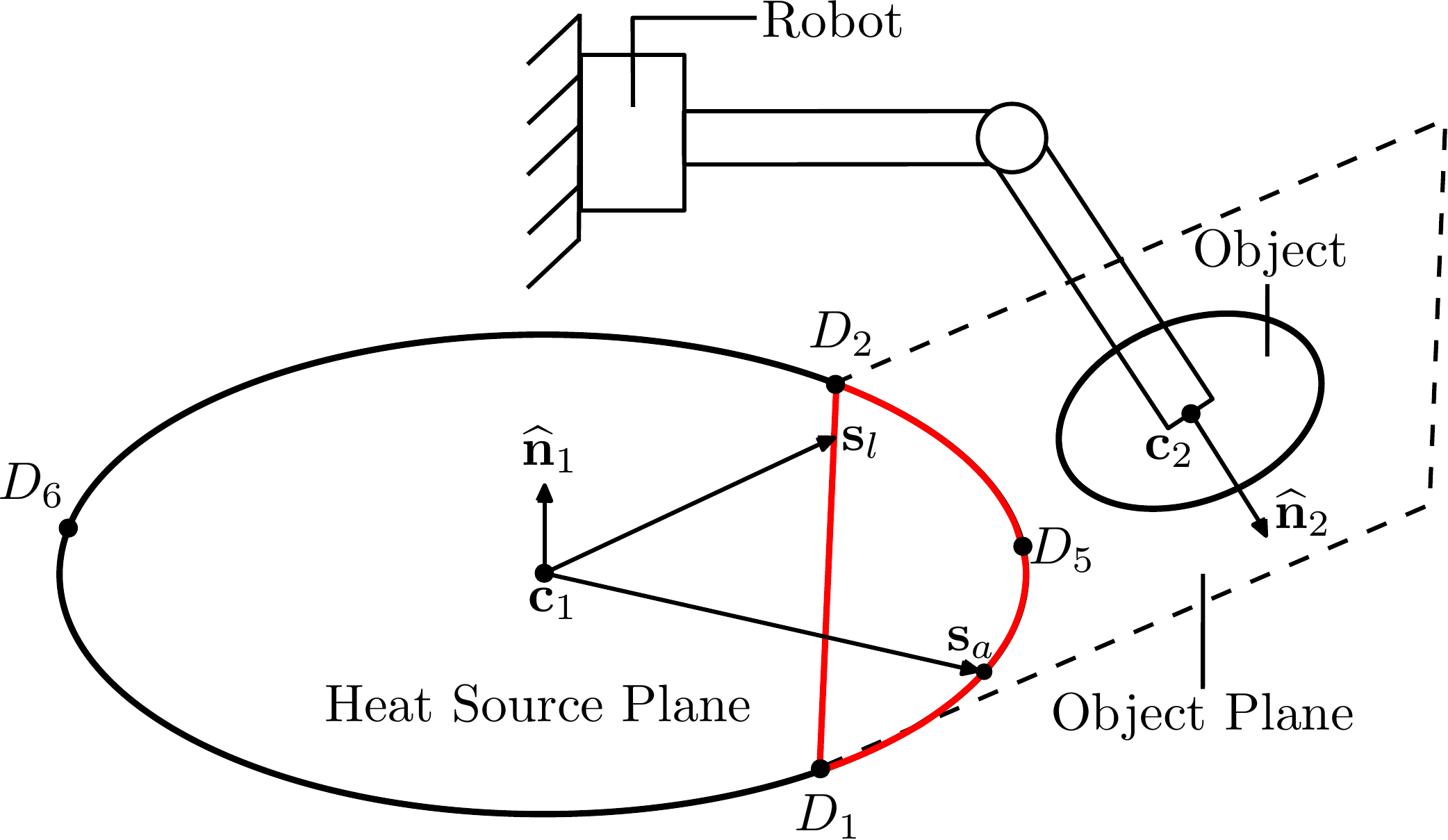}
				\caption{Conceptual representation of a  self-obstruction case.}
				\label{fig:self obstruction}
\end{figure}
In this section, we extend the parallel surfaces problem to a 6-DOF scenario, where the end-effector configuration is now defined as $\mathbf x = [p_1,p_2,p_3,\theta_{x}, \theta_{y}, \theta_{z}{]}^\T$, for $\theta_{i}$ as the angles around the object's coordinate system (see Fig. \ref{fig:VF}).
We denote by $\mathbf R$ the 3D rotation matrix corresponding to this relative orientation.
Note that in some configurations of this non-parallel case, radiation from a source's region cannot reach the front side of object's surface (hence, will not contribute to the heat inflow). We refer to this problem as \emph{self-obstruction}.

To model this situation, let us denote the object plane as $D_{1}D_{2} \mathbf c_{2}$, for $D_{1}$ and $D_{2}$ as the intersections with the bounded source plane.
This setup is depicted in Fig. \ref{fig:self obstruction}, where the heat source is divided into two surfaces: the red surface composed\footnote{The symbol $\arc{abc}$ denotes the arc that passes through the points $a$, $b$ and $c$.} of $\arc{D_{1}D_{5}D_{2}}D_{2}D_{1}$ and the black surface composed of $\arc{D_{2}D_{6}D_{1}}D_{1}D_{2}$. 
The black surface only ``sees'' the object's backside (i.e. the robot's gripper), thus, is omitted from the following calculation of $F_{21}$:
\begin{equation}
\label{eq:new F21}
F_{21}=\frac{1}{2 \pi A_{2}} \left( \oint_{\Gamma_{a}} \oint_{\Gamma_{2}} \ln s_{a}\diff \mathbf s_{2} \diff \mathbf s_{a} + \oint_{\Gamma_{l}} \oint_{\Gamma_{2}} \ln s_{l}\diff \mathbf s_{2} \diff \mathbf s_{l} \right) 
\end{equation}
where $\Gamma_{a}$ denotes the arc $\arc{D_{1}D_{5}D_{2}}$ and $\Gamma_{l}$ the line $D_{1}D_{2}$. 
To derive its respective position vectors $\mathbf{s}_{a}$ and $\mathbf{s}_{l}$, we compute the vector $\mathbf{n}_{2} = [n_{2}^1 , n_{2}^2 , n_{2}^3 {]}^\T$ normal to the object plane as:
\begin{equation}
    \mathbf{n}_{2} = \mathbf R \begin{bmatrix} 0 & 0 & -1 \end{bmatrix}^\T 
    \label{eq:normal_with_R}
\end{equation}
whose plane equation satisfies:
\begin{equation}
\label{eq:object plane equation}
    n_{2}^1 \left(x-p_{1}\right)+  n_{2}^2 \left(y-p_{2}\right)+ n_{2}^3 \left(z-p_{3}\right)=0
\end{equation}
To find the intersection with the plane, we use \eqref{eq:normal_with_R} and substitute $x=r_{1} \cos \varphi, $ $y=r_{1} \sin \varphi$ and $z=0$ into \eqref{eq:object plane equation}, for $\varphi$ as a variable angle. Self-obstruction occurs when there exist two solutions $\varphi_{1}$ and $\varphi_{2}$; After some algebraic operations, the arc and line parametric vectors are obtained from:
\begin{equation}
\label{eq:sasl}
    \mathbf{s}_{a}=\begin{bmatrix}
    r_{1}\cos{\varphi}\\r_{1}\sin{\varphi}\\0
    \end{bmatrix},~
    \mathbf{s}_{l}=\begin{bmatrix}
    x_{l}\\k_{l}(x_{l}-r_{1}\cos{\varphi_{2}})+r_{1}\sin{\varphi_{2}}\\0
    \end{bmatrix}
\end{equation}
for a distance range $x_{l} \in [r_{1} \cos {\varphi_{2}}, r_{1} \cos {\varphi_{1}}]$, an angle range $\varphi \in [ \varphi_{1},\varphi_{2}]$, and a slope $k_{l}= \frac{ \sin{\varphi_{2}}-\sin{\varphi_{1}}}{\cos{\varphi_{2}}-\cos{\varphi_{1}}}$ of the line $D_{1}D_{2}$.
The parametric vector on the object's contour is computed as: 
\begin{equation}
\label{eq: new s2}
    \mathbf{s}_{2} = \mathbf R 
    \begin{bmatrix}
    {r}_{2}\cos{\omega}_{2} & {r}_{2}\sin{\omega}_{2} & 0
    \end{bmatrix}^\T
    +
    \begin{bmatrix}
    p_1 & p_2 & p_3
    \end{bmatrix}^\T
\end{equation}
As with the parallel surface case, we compute the arc and line distances $s_{a}=\| \mathbf{s}_{2}-\mathbf{s}_{a} \|$ and $s_{l}=\| \mathbf{s}_{2}-\mathbf{s}_{l} \|$. 
The 6-DOF view factor for the self-obstruction case is as follows:
\begin{multline}
F_{21}=\frac{1}{2 \pi \mathrm{A}_{2}} \int_{\varphi_{1}}^{\varphi_{2}} \int_{2 \pi}^{0}\ln s_{a} \diff{\omega}_{1} \diff{\varphi}\\
+\frac{1}{2 \pi \mathrm{A}_{2}} \int_{r_{1}\cos{\varphi_{2}}}^{r_{1}\cos{\varphi_{1}}} \int_{2 \pi}^{0}\ln s_{l} \diff{\omega}_{1} \diff{x_{l}.}
\end{multline}
With this expression, we can derive a similar GTM model $\dot{v}=\mathbf{l}\cdot{\mathbf{u}}-4\lambda_{2}{T_{2}}^3v$, where $\mathbf u = \dot {\mathbf x} \in \mathbb R^6$ and the interaction matrix is $\mathbf{l}= \lambda_{1} \frac{\partial F_{21}}{\partial \mathbf x}^\T \in\mathbb R^6$.
For this 6-DOF case, we use the following numerical differentiation method to approximate $\mathbf{l}$:
\begin{equation}
\label{nummerical approximation of L_inv}
    \mathbf{l}=\lambda_{1}
    \begin{bmatrix}
    \cfrac{F_{21}(p_{1}+\diff p_{1},p_{2},...,\theta_{z})-F_{21}(\mathbf x)}{\diff p_{1}}\\
    \vdots\\
    \cfrac{F_{21}(p_{1},p_2,\ldots,\theta_{z}+\diff\theta_{z})-F_{21}(\mathbf x)}{\diff \theta_{z}}
    \end{bmatrix}
\end{equation}
Parallel programming techniques can be applied to achieve real-time capabilities, where every element of $\mathbf{l}$ is simultaneously calculated by an independent process.

\begin{remark}
When the object plane and the bounded source plane do not intercept, we classify it either as a common (non-self-obstruction) case or as a complete self-obstruction case. 
This classification can be done by checking whether the ray $\mathbf{n}_{2}$ intercepts the bounded source plane. 
For the common case, the entire contour of the heat source will be used to calculate $F_{21}$. For the complete self-obstruction case, $F_{21}=0$.
\end{remark}

\subsection{Thermal Servoing Model with Multiple Objects}
Here, we consider the case where the robot rigidly grasps $N$ objects with its end-effector, and independently regulates the temperature of each object. 
For this situation, we assume that heat exchange amongst the objects is negligible, therefore, the derivation of the $N$-object interaction matrix $\mathbf{L} \in \mathbb{R}^{N\times n}$ (where $n$ is the number of DOF of the robot) is analogous to the previous sections and is simply constructed with $N$ vectors $\mathbf l_i$ (defined for the $i$th object) as follows:
\begin{equation}
    \label{eq:Jacobian Matrix}
    \mathbf{L}=\begin{bmatrix}
    \mathbf{l}_{1} & \mathbf{l}_{2} & \cdots & \mathbf{l}_{N}
    \end{bmatrix}^\T. 
\end{equation}
\begin{figure}
				\centering
				\includegraphics[width =0.7 \columnwidth]{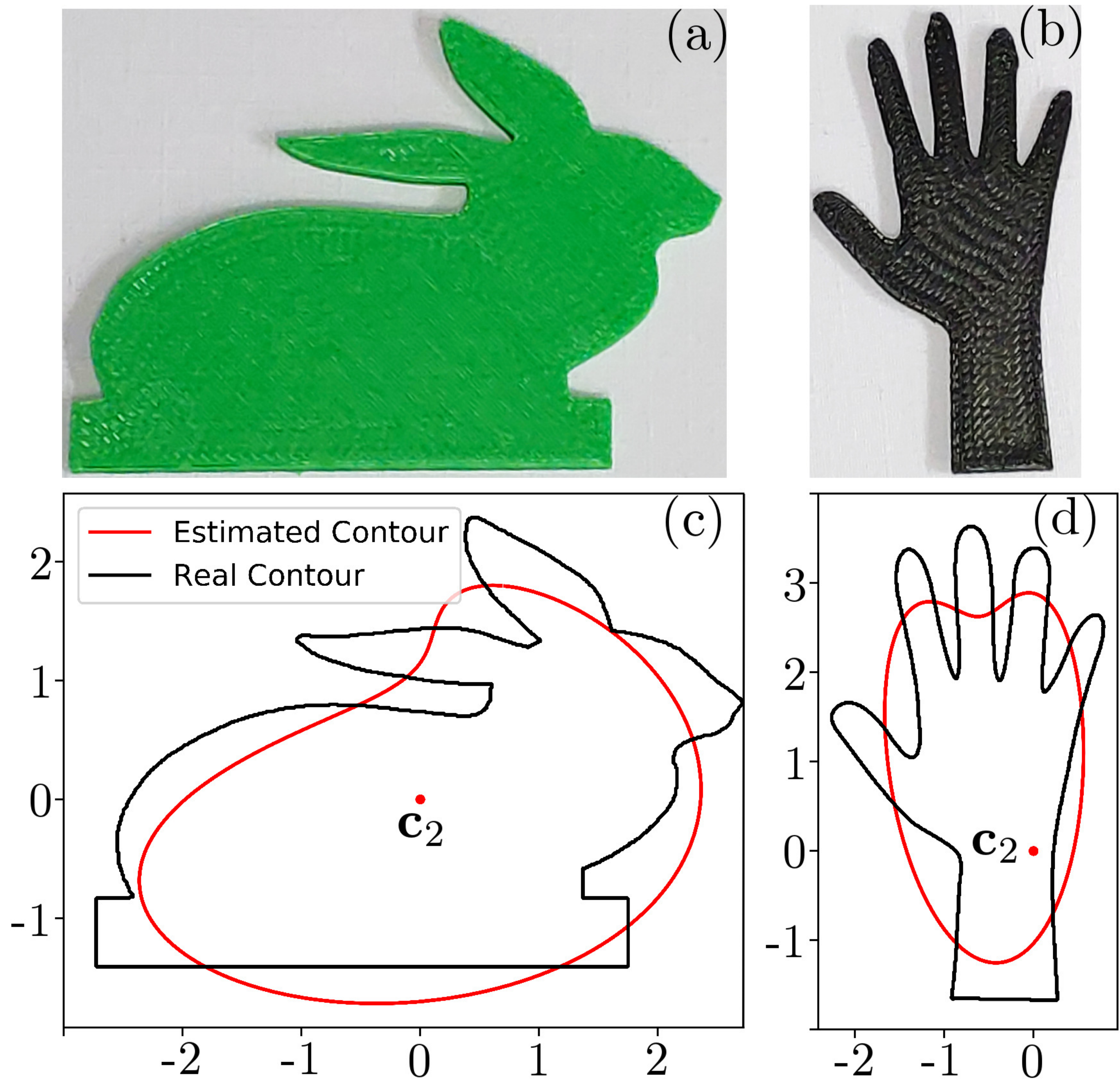}
				\caption{Contours of two objects approximated by a truncated Fourier series with $5$ harmonics.}
				\label{fig:bunny-hand real-Fourier}
\end{figure}
To ensure the independent control of each feedback temperature, we limit the number of objects to be fewer than the number of robot DOF, i.e. $N\le n$. 
For this multi-object system, we construct the following structures:
\begin{align}
    \boldsymbol{\tau}&=\begin{bmatrix}
    T_{2}^1&T_{2}^2&\cdots&T_{2}^N
    \end{bmatrix}^\T \in\mathbb R^N\\
    {\mathbf T} &= \operatorname{diag} \left((T_{2}^{1})^3,(T_{2}^{2})^3,\cdots,(T_{2}^{N})^3\right) \in\mathbb R^{N\times N} \\
    \mathbf{v}&=\begin{bmatrix}
    v^1&v^2&\cdots&v^N
    \end{bmatrix}^\T \in\mathbb R^N
\end{align}
The constant thermophysical parameters $\lambda_{2}$ are defined for $i$th object as $\lambda_{2}^{i}$, and are grouped into the constant matrix:
\begin{equation}
\boldsymbol{\Lambda}=
\operatorname{diag} (\lambda_{2}^1,\lambda_{2}^2\cdots,\lambda_{2}^N )\in \mathbb R^{N\times N} 
\end{equation}
With all these terms, the geometric-thermal-motor model can be extended to a multi-object case:
\begin{equation}
    \dot{\mathbf{v}}=\mathbf{L}{\mathbf{u}}-4{\mathbf T}\boldsymbol{\Lambda}\mathbf{v}
    \label{eq:multiple object dynamic system}
\end{equation}

\subsection{Irregularly Shaped Surfaces}
\label{section:Arbitrary Surfaces at Arbitrary Configurations}

In this section, we propose an efficient method to calculate view factors (which, in turn, are needed for calculating interaction matrices $\mathbf L$) of irregularly shaped flat surfaces. 
The core concept is to use truncated Fourier series for approximating the parametric position vectors $\mathbf{s}_{1}$ and $\mathbf{s}_{2}$.
This method can be used to model both the object and the heat source.

To see this, let us consider that the robot manipulates an object with an irregular contour (e.g. as in Fig. \ref{fig:bunny-hand real-Fourier}).
We can approximate the parametric position vector of $\mathbf{s}_{2}$ with the following approach \cite{tolstov2012fourier}:
\begin{equation}
\label{fourier s2}
   \mathbf{s}_{2}= \begin{bmatrix}
    p_{1}+\sum_{j=-F}^{F} \beta_{j}\sin(2 \pi j \phi)+b_{n}\cos(2 \pi j \phi)\\
    p_{2}+\sum_{j=-F}^{N} \beta_{j}\cos(2 \pi j \phi)-b_{n}\sin(2 \pi j \phi)\\
    p_{3}
    \end{bmatrix}
\end{equation}
for $\phi \in [0,1)$ as the arc-length along the contour, and $\beta_j$ and $b_j$ constant coefficients (which can be computed as in \cite{navarro2018fourier}). The object's contour is approximated with $2F+1$ harmonic terms.
The differential change can then be computed as:
\begin{equation}
    \diff{\mathbf{s}}_{2}=\begin{bmatrix}
    \sum_{j=-F}^{F}  2\pi j \left(\beta_{j}\cos(2 \pi j \phi)-b_{j}\sin(2 \pi j \phi)\right)\\
    \sum_{j=-F}^{F} -2\pi j \left(\beta_{j}\sin(2 \pi j \phi)+b_{j}\cos(2 \pi j \phi) \right)\\
    0
    \end{bmatrix} \diff{\phi}
\end{equation}
The thermal interaction matrix $\mathbf L$ can now be easily derived by following procedures similar to the ones presented in the previous sections.
In Appendix B, we present a numerical study comparing the performance of this method with the standard discrete surface integral \cite{muneer2015finite}.

\section{CONTROLLER DESIGN}
\label{section:controller design}

\begin{problem}
Given a \emph{constant} temperature reference vector $\boldsymbol{\tau}^* = [{T}^{*1},\ldots,{T}^{*N}{]}^\T \in \mathbb{R}^{N}$, design a velocity-based motion controller $\mathbf{u}$ that asymptotically minimizes the feedback error $\Delta \boldsymbol{\tau}= \boldsymbol{\tau}-\boldsymbol{\tau}^*$ for all $N$ objects.
\end{problem} 

To solve this problem, in this section we propose two methods: a model-based controller (conceptually depicted in Fig.~\ref{fig:model_based_controller.pdf}) and an adaptive controller (shown in Fig.~\ref{fig:adaptive_controller}). 
At the end of this section, we discuss target feasibility.

\subsection{Model-Based Controller}
\label{sec: Model-Based Controller}

\begin{figure}[t]
				\centering
				\includegraphics[width = 0.75 \columnwidth]{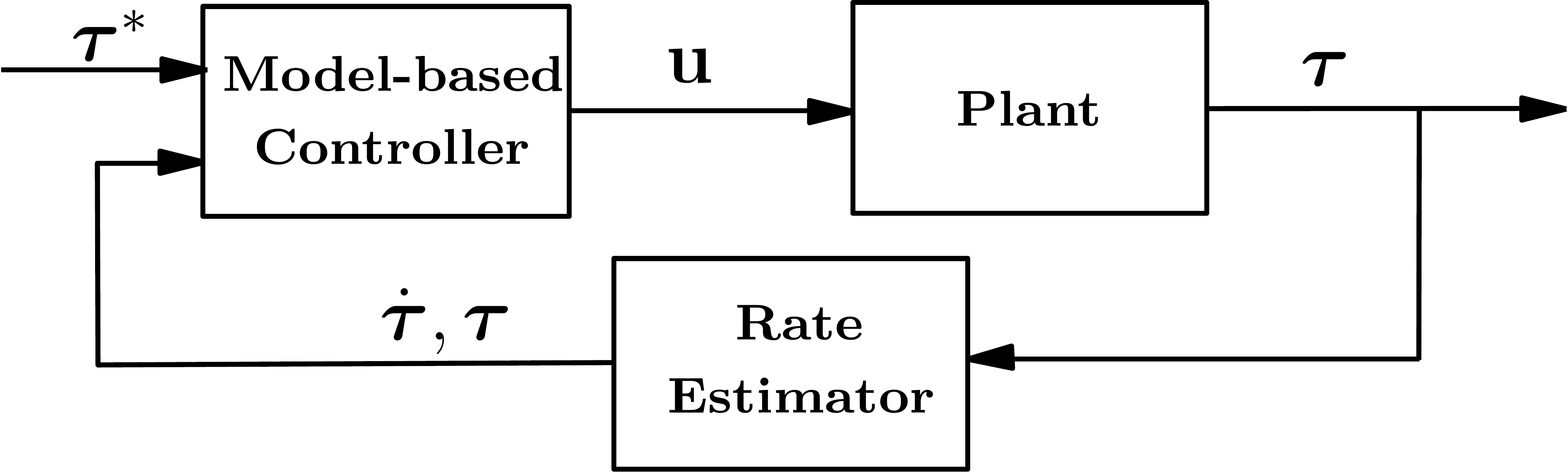}
				\caption{Schematic representation of the model-based controller.}
				\label{fig:model_based_controller.pdf}
\end{figure}

To automatically regulate the feedback temperature of $N$ objects, we design the following velocity control input:
\begin{equation}
\mathbf{u}=\mathbf{L}^{+}(-D\mathbf{v}-K\Delta \boldsymbol{\tau}+4{\mathbf T} \boldsymbol{\Lambda}\mathbf{v})
\label{control input}
\end{equation}
where $\mathbf{L}^{+}=\mathbf{L}^\T\left(\mathbf{L} \mathbf{L}^\T\right)^{-1}$ is the Moore-Penrose pseudoinverse of $\mathbf{L}$ \cite{Book:Nakamura1991}, and $D>0$ and $K>0$ are control gains. 

\begin{prop}
Consider that thermodynamic parameters in \eqref{eq:multiple object dynamic system} are accurately known.
For this situation, the control input \eqref{control input} enforces a stable closed-loop system which asymptotically minimizes $\|\Delta \boldsymbol{\tau}\|$.
\end{prop}

\begin{proof}
Substituting \eqref{control input} into the nonlinear dynamic system \eqref{eq:multiple object dynamic system}, yields the following closed-loop system:
\begin{equation}
\label{eq:dot_v}
    \dot{\mathbf{v}}=-D\mathbf{v}-K\Delta \boldsymbol{\tau}.
\end{equation}
Consider the quadratic Lyapunov function
\begin{equation}
\mathcal Q (\mathbf v,\Delta\boldsymbol\tau) = \frac{1}{2} \| \mathbf{v} \|^2 +\frac{1}{2}K \|\Delta \boldsymbol{\tau} \|^2
\end{equation}
whose time derivative along trajectories of \eqref{eq:dot_v} yields
\begin{equation}
\begin{split}
    \label{eq:mb controller step 1}
    \dot{\mathcal Q}(\mathbf v,\Delta\boldsymbol\tau) &= \mathbf{v}^\T\dot{\mathbf{v}}+K\Delta\boldsymbol\tau^\T {\mathbf{v}} \\
    &=-D\|\mathbf{v}\|^2
\end{split}
\end{equation}
which shows that the energy function is non-increasing, i.e. $\dot {\mathcal Q} \le 0$, thus, the closed-loop system is stable.
By applying the Krasovskii-LaSalle principle \cite{vidyasagar2002nonlinear}, the asymptotic minimization of $\|\Delta\boldsymbol\tau\|$ can be proved.
\end{proof}

\begin{remark}
In our proposed method, the terms $\mathbf T$ and $\boldsymbol{\tau}$ in the controller \eqref{control input} can be directly obtained from real-time sensor measurements. 
Yet, to implement the variable $\mathbf{v}$, we use a rate estimation algorithm based on polynomial fitting with sliding windows \cite{titchener2015calculation} (see Appendix C for details). 
\end{remark}

\subsection{Adaptive Controller}
\label{section:adaptive controller}
In the above model-based controller, we assume that the object's thermophysical properties are exactly known. 
However, due to the differences in material and surface conditions, it is hard to determine the true values. 
In this subsection, we propose an adaptive controller which updates the unknown parameters within the servo-loop. 
To this end, we start by introducing the unknown parameters $a_{1}=\frac{1}{\lambda_{1}}$ and ${a}_{2} = \frac{\lambda_{2}}{\lambda_{1}}$, which are well-defined since $\lambda_{1} > 0$; We use the superscripts ${a}_{1}^{i}$ and ${a}_{2}^{i}$ to distinguish them between different objects. 
With these parameters, we construct the following \emph{constant} vector $\mathbf a_{1,2}\in\mathbb R^N$ and matrix $\mathbf A_{1,2}\in\mathbb R^{N\times N}$ structures:
\begin{align}
\mathbf{a}_{1}&= \begin{bmatrix} a_{1}^{1} & \cdots & a_{1}^N \end{bmatrix}^\T,
\qquad \mathbf{A}_{1} =\operatorname{diag}(\mathbf a_1)>0, \nonumber \\
\mathbf{a}_{2} &= \begin{bmatrix} a_{2}^{1} & \cdots & a_{2}^N \end{bmatrix}^\T, 
\qquad\mathbf{A}_{2}=\operatorname{diag}(\mathbf a_2)>0.
\end{align}
By applying the dynamic expression \eqref{eq:dynamic system} to the multi-object case and dividing it by $\lambda_{1}^{i}$ for each $i$th object, we obtain:
\begin{equation}
    \mathbf{A}_{1}\dot{\mathbf{v}}+4 \mathbf T {\mathbf{A}}_{2}\mathbf{v}=\mathbf{J}\mathbf{u}
    \label{eq:new dynamic equation}
\end{equation}
for an new interaction matrix
\begin{equation}
    \mathbf{J}=\begin{bmatrix}
    \mathbf{l}_{1}/\lambda_{1}^{1}&\mathbf{l}_{2}/\lambda_{1}^{2}&\cdots&\mathbf{l}_{N}/\lambda_{1}^{N}
    \end{bmatrix}^\T \in \mathbb{R}^{N \times 6}
\end{equation}
which is independent from the unknown thermophysical parameters\footnote{An analogous result to the depth-independent interaction matrix in \cite{Journals:Liu2006}} and is entirely computed with the $N$ gradients of the view factors. 
To design the adaptive controller, it is useful to introduce the combined thermal error vector $\boldsymbol{\zeta}= [\zeta^{1},\ldots,\zeta^{N}{]}^\T\in\mathbb R^N$ defined as:
\begin{equation}
  \boldsymbol{\zeta}= \Delta \dot{\boldsymbol{\tau}}+\mu\Delta \boldsymbol{\tau}=\mathbf{v}+\mu \Delta \boldsymbol{\tau}  
\end{equation}
for $\mu>0$ as an arbitrary feedback gain. 
To control the $N$ object temperatures, we design the following velocity input:
\begin{equation}
\label{eq:adaptive control input}
    \mathbf{u}=\mathbf{J}^{+}(-\mu\widehat{\mathbf{A}}_{1}\mathbf{v}-K\boldsymbol{\zeta}+4 \mathbf T \widehat{\mathbf{A}}_{2}\mathbf{v})
\end{equation}
where the elements of the adaptive diagonal matrices $\widehat{\mathbf A}_i = \operatorname{diag}(\widehat{\mathbf a}_i)\in\mathbb R^{N\times N}$ are computed with the update rules:
\begin{align}
\dot{\widehat{\mathbf{a}}}_1 &= \gamma_{1}\mu\begin{bmatrix}
v^{1}\zeta^{1} & \ldots & v^{N}\zeta^{N}
\end{bmatrix}^{\T} \in\mathbb R^N
\label{update rule 1}\\
\dot{\widehat{\mathbf{a}}}_2 &= -4\gamma_{2}\begin{bmatrix}
v^1\zeta^{1}(T_{2}^{1})^3 & \ldots & v^N\zeta^{N}(T_{2}^{N})^3
\end{bmatrix}^{\T} \in\mathbb R^N
\label{update rule 2}
\end{align}
where the positive scalars $\gamma_{i}>0$ are used for tuning algorithm's learning rate.

\begin{figure}[t]
                \centering
				\includegraphics[width = 0.95 \columnwidth]{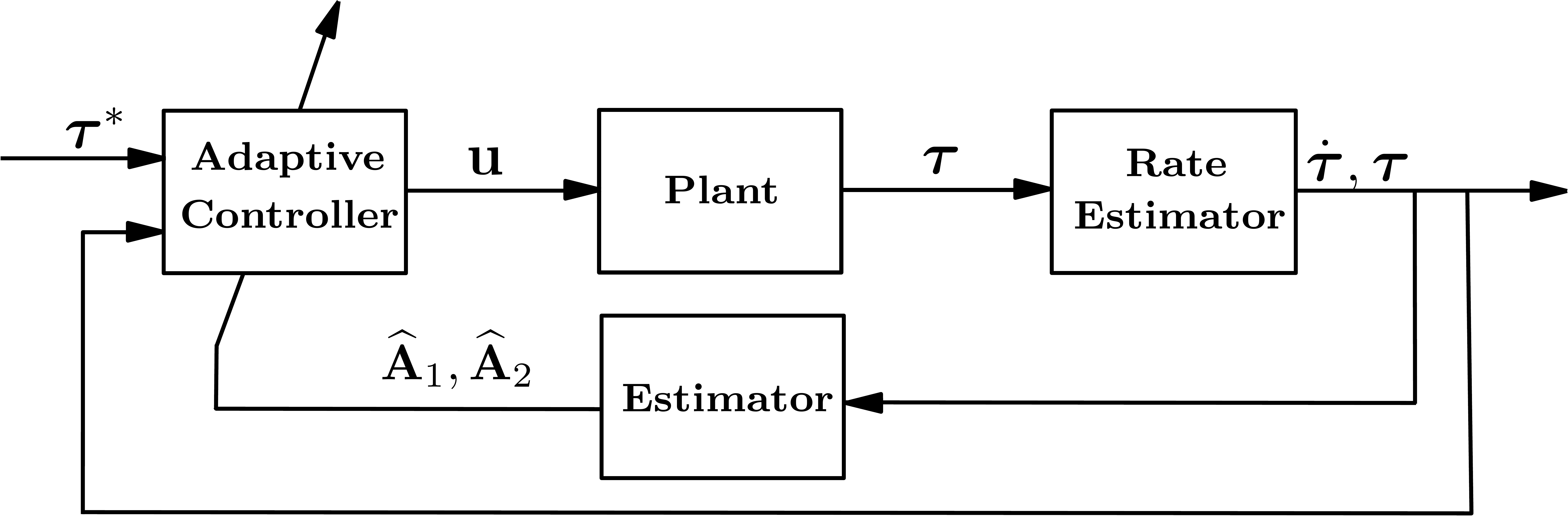}
				\caption{Schematic representation of the adaptive controller.}
				\label{fig:adaptive_controller}
\end{figure}

\begin{prop}
The adaptive controller \eqref{eq:adaptive control input} with update rules \eqref{update rule 1}--\eqref{update rule 2} guarantees a bounded estimation of the unknown parameters $\mathbf a_{1}$ and $\mathbf a_{2}$, and  the asymptotic minimization of the thermal error $\|\Delta \boldsymbol{\tau}\|$.
\end{prop}

\begin{proof}
Substitution of \eqref{eq:adaptive control input} into \eqref{eq:new dynamic equation} yields:
\begin{equation}
    \mathbf{A}_{1}\dot{\mathbf{v}} + 4 \mathbf T {\mathbf{A}}_{2}\mathbf{v} = -\mu\widehat{\mathbf{A}}_{1}\mathbf{v}-K\boldsymbol{\zeta} + 4 \mathbf T \widehat{\mathbf{A}}_{2}\mathbf{v}
    \label{first_closed-loop_prop2}
\end{equation}
By adding $\mu\mathbf{A}_{1} \mathbf{v}$ to both sides of \eqref{first_closed-loop_prop2}, noting that $\dot{\boldsymbol{\zeta}}=\dot{\mathbf{v}}+\mu\mathbf{v}$, and performing some algebraic operations we can obtain:
\begin{equation}
\label{eq:new zeta}
    \mathbf{A}_{1}\dot{\boldsymbol{\zeta}}+K\boldsymbol{\zeta}=-\mu\widetilde{\mathbf{A}}_{1}\mathbf{v} + 4 \mathbf T \widetilde{\mathbf{A}}_{2}\mathbf{v}.
\end{equation}
for error matrices $\widetilde{\mathbf{A}}_{i} = \widehat{\mathbf A}_i - \mathbf A_i = \operatorname{diag}(\widetilde{\mathbf a}_i)$, with error vectors $\widetilde{\mathbf{a}}_{i} = \widehat{\mathbf a}_i - \mathbf a_i$.
To analyze the stability of the closed-loop dynamical system \eqref{update rule 1}--\eqref{update rule 2} and \eqref{eq:new zeta}, we introduce the following Lyapunov function:
\begin{equation}
    \mathcal H(\boldsymbol\zeta,\widetilde{\mathbf a}_1,\widetilde{\mathbf a}_2) = \frac{1}{2} \boldsymbol{\zeta}^\T\mathbf{A}_{1}\boldsymbol{\zeta} + \frac{1}{2\gamma_1} \|{\widetilde{\mathbf{a}}}_{1}\|^2 + \frac{1}{2\gamma_2} \|{\widetilde{\mathbf{a}}}_{2} \|^2
    \label{eq:adaptive lyapunov function}
\end{equation}
whose time derivative along \eqref{update rule 1}--\eqref{update rule 2} and \eqref{eq:new zeta} yields
\begin{equation}
\begin{split}
\dot{\mathcal H}(\boldsymbol\zeta,\widetilde{\mathbf a}_1,\widetilde{\mathbf a}_2) &= \boldsymbol{\zeta}^{\T}\mathbf{A}_{1}\dot{\boldsymbol{\zeta}}+\frac{1}{\gamma_{1}}\dot{\widehat{\mathbf{a}}}_{1}^{\T}\widetilde{\mathbf{a}}_{1}+\frac{1}{\gamma_{2}}\dot{\widehat{\mathbf{a}}}_{2}^{\T}\widetilde{\mathbf{a}}_{2}\\
&= - K \|\boldsymbol\zeta\|^2
\label{eq:derive dotH}
\end{split}
\end{equation}
which shows that the energy function is non-increasing, i.e. $\dot{\mathcal H}\le 0$, thus, the parameter estimation errors $\widetilde{\mathbf a}_i$ are bounded.
Asymptotic stability of $\Delta\boldsymbol\tau$ directly follows by applying the Krasovskii-LaSalle principle \cite{vidyasagar2002nonlinear}.
\end{proof}

\subsection{Target Feasibility}
\label{section:Target Feasibility}
In previous sections, we proved that $\|\Delta \boldsymbol{\tau}\|$ can be asymptotically minimized by two automatic controllers. 
However, it is not guaranteed that such error can be enforced to zero. 
Failure cases are caused by the choice of unfeasible target temperatures: Intuitively, if targets are set to too high or too low, they might be physically unachievable; In addition, for objects fixed to the same end-effector, the difference range between their target temperatures is constrained by the fixed spatial relationship between the objects. 
In this section, we analyze two necessary but not sufficient conditions to ensure the feasibility of the targets. 
Failed experimental results are analyzed accordingly in Section \ref{sec: Result Unfeasible Thermal Targets}.

Consider a simple case with two objects, object 1 and object 2, fixed to the end-effector (the extension to $N$ object is straightforward). 
For one of the objects, recall the thermal-geometric relation $\eqref{eq:final temperature rate}$ and rewrite it as follows:
\begin{equation}
    \label{eq:v for T_{v0}}
    v=-\lambda_{2}{T_{2}}^4+\lambda_{1}{F}_{21}(\mathbf{x}_{o})+\lambda_{3}
\end{equation}
where $\mathbf{x}_{o}$ denotes an object configuration. 
Let us assume there exists a temperature $T_{2}=T_{v0}$ that makes the rate $v=0$. 
As temperature is always non-negative, $T_{v0}$ can be solved as:
\begin{equation}
\label{eq:steady state temperature}
    T_{v0}(F_{21})= \left((\lambda_{1}{F}_{21}+\lambda_{3}) / \lambda_{2} \right)^{\frac{1}{4}}
\end{equation}
Since the parameters $\lambda_{i}>0$ are all positive and $F_{21} \in [0,1)$, $T_{v0}$ always exists. %; It is easy to show that when $T_{2}>T_{v0}$, $v<0$, and when $T_{2}<T_{v0}$, $v>0$. 
$T_{v0}$ represents the steady state temperature at $\mathbf{x}_{o}$. 
Note that $T_{v0}$ is a function of $F_{21}$ and that $\partial T_{v0} / \partial F_{21}>0$ is always positive. 
Thus, the minimum value of $T_{v0}$ is determined when $F_{21}=0$ as:
\begin{equation}
    \operatorname{min}(T_{v0})=(\lambda_{3}/\lambda_{2})^{\frac{1}{4}}=(\alpha_{2}T_{3}^{4} / \varepsilon_{2})^{\frac{1}{4}}
\end{equation}
According to Kirchhoff's law of thermal radiation \cite{Fundamentals_of_heat_and_mass_transfer}, at thermodynamic equilibrium, $\alpha_{2}=\varepsilon_{2}$. 
Thus, the minimum is:
\begin{equation}
    \label{eq:Tv0_min}
    \operatorname{min}(T_{v0})=T_{3}
\end{equation}
When $F_{21}(\mathbf{x}_{o}) \to 1$, the maximum value of $T_{v0}$ approaches:
\begin{equation}
    \label{eq:Tv0_max}
    \operatorname{max}(T_{v0}) \to ( (\lambda_{1}+\lambda_{3}) / \lambda_{2} )^{\frac{1}{4}} = (\alpha_{2}\varepsilon_{1}T_{1}^{4}/\varepsilon_{2})^{\frac{1}{4}}=\varepsilon_{1}^{\frac{1}{4}}T_{1}
\end{equation}
From \eqref{eq:Tv0_min}--\eqref{eq:Tv0_max}, we derive the \emph{first} boundary value condition:
\begin{equation}
    T^{1*}, T^{2*}\in [T_{3},\varepsilon_{1}^{\frac{1}{4}}T_{1})
\end{equation}

Now we discuss the limitation of the difference between target temperatures $|\delta T^{*}|=|T^{1*}-T^{2*}|$.
We denote the configuration of the object 1 and object 2 by $\mathbf{x}_{o1}=\mathbf{x}+\Delta \mathbf{x}_{1}$ and $\mathbf{x}_{o2}=\mathbf{x}+\Delta \mathbf{x}_{2}$, respectively, where $\Delta \mathbf{x}_{1},\Delta \mathbf{x}_{2}$ are constant displacement vectors determined by the arrangement of objects.
Their corresponding view factor are denoted by $F_{21}(\mathbf{x}_{o1})$ and $F_{21}(\mathbf{x}_{o2})$, and its steady-state temperatures by $T_{v0}^{1}$ and $T_{v0}^{2}$. 
According to $\eqref{eq:steady state temperature}$,  $|\Delta T_{v0}|=|T_{v0}^{1}-T_{v0}^{2}|$ can be expressed as:
\begin{equation}
    |\Delta T_{v0}(\mathbf{x}_{o1},\mathbf{x}_{o2})|=|\chi_{1}(F_{21}(\mathbf{x}_{o1}))-\chi_{2}(F_{21}(\mathbf{x}_{o2}))|
\end{equation}
with functions $\chi_{1}(F_{21})$ and $\chi_{2}(F_{21})$ defined as:
\begin{equation}
    \chi_{1}(F_{21})= \left(\frac{\lambda_{1}^{1}{F}_{21}+\lambda_{3}^{1}}{\lambda_{2}^{1}}\right)^{\frac{1}{4}},~ \chi_{2}(F_{21})= \left(\frac{\lambda_{1}^{2}{F}_{21}+\lambda_{3}^{2}}{\lambda_{2}^{2}}\right)^{\frac{1}{4}}
\end{equation}
where $\lambda_{i}^{1}$ and $\lambda_{i}^{2}$ are the thermophysical parameters of the two objects.
Note that for the continuous function $\Delta T_{v0}(\mathbf{x}+\Delta \mathbf{x}_{1},\mathbf{x}+\Delta \mathbf{x}_{2})$, where $\mathbf{x}\in \mathbb W$ for $\mathbb{W}$ as the bounded workspace and $\Delta \mathbf{x}_{j}$ as constant vectors, there must exist a minimum value $\operatorname{min}(\Delta T_{v0})=\Delta T_{v0}(\mathbf{x}^{min})$ and a maximum value $\operatorname{max}(\Delta T_{v0})=\Delta T_{v0}(\mathbf{x}^{max})$ which encompass all possible values of $\Delta T_{v0}$, where $\mathbf{x}^{min}$ and $\mathbf{x}^{max}$ are the end-effector configurations corresponding to the two extreme cases. 
The \emph{second} condition for feasible target temperatures is:
\begin{equation}
    \delta T^{*} \in [\operatorname{min}(\Delta T_{v0}),\operatorname{max}(\Delta T_{v0})]
\end{equation}
A numerical (geometric) interpretation of $\mathbf{x}^{min}$ and $\mathbf{x}^{max}$ will be discussed in Section $\ref{sec: Result Unfeasible Thermal Targets}$.

\section{RESULTS}
\label{section:results}

\subsection{Experimental Setup}
\label{sec:set up}
\begin{figure}
				\centering
				\includegraphics[width =0.8 \columnwidth]{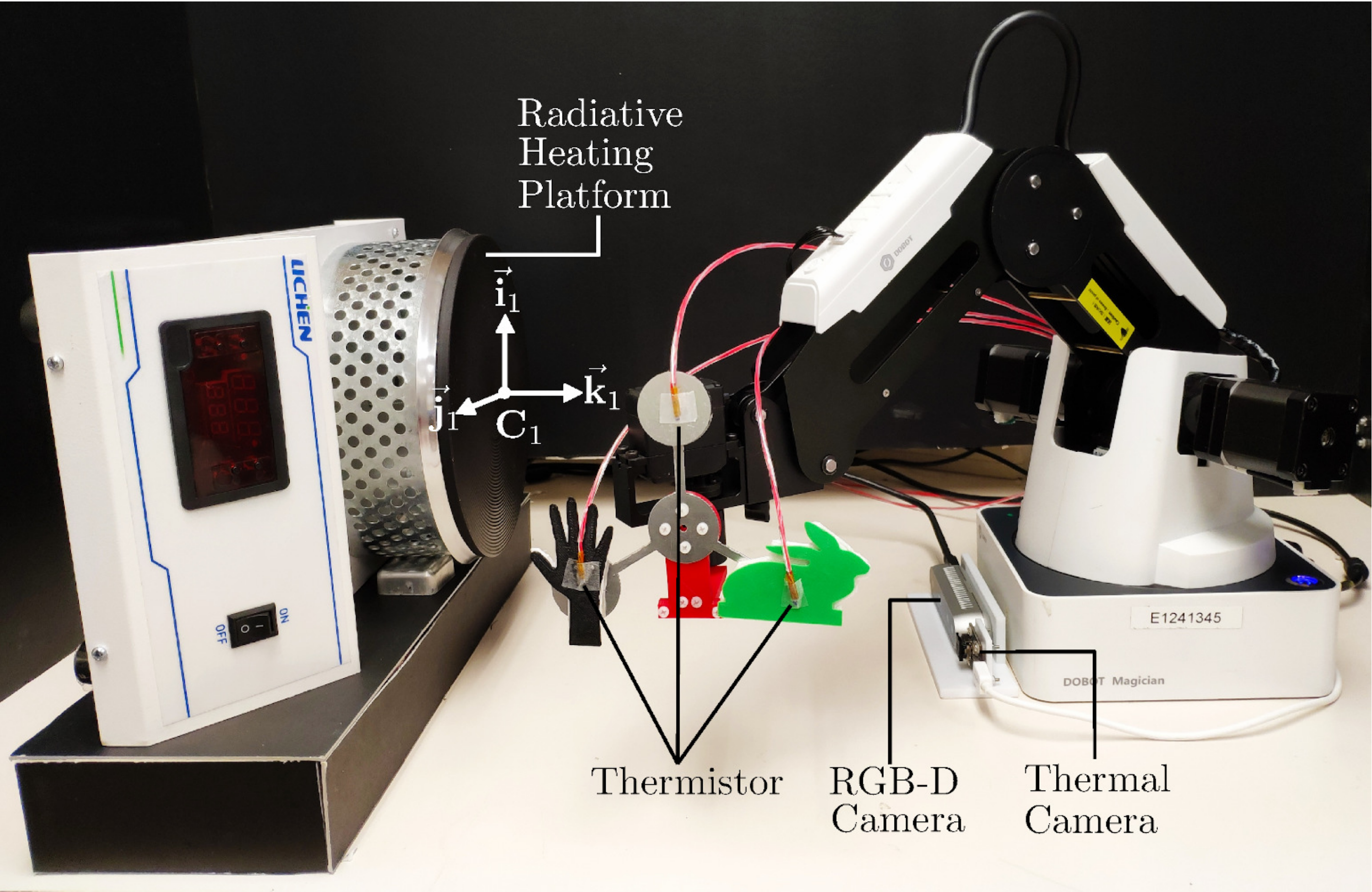}
				\caption{Experimental setup.}
				\label{fig: Set up}
\end{figure}

\begin{figure}
				\centering
				\includegraphics[width =1 \columnwidth]{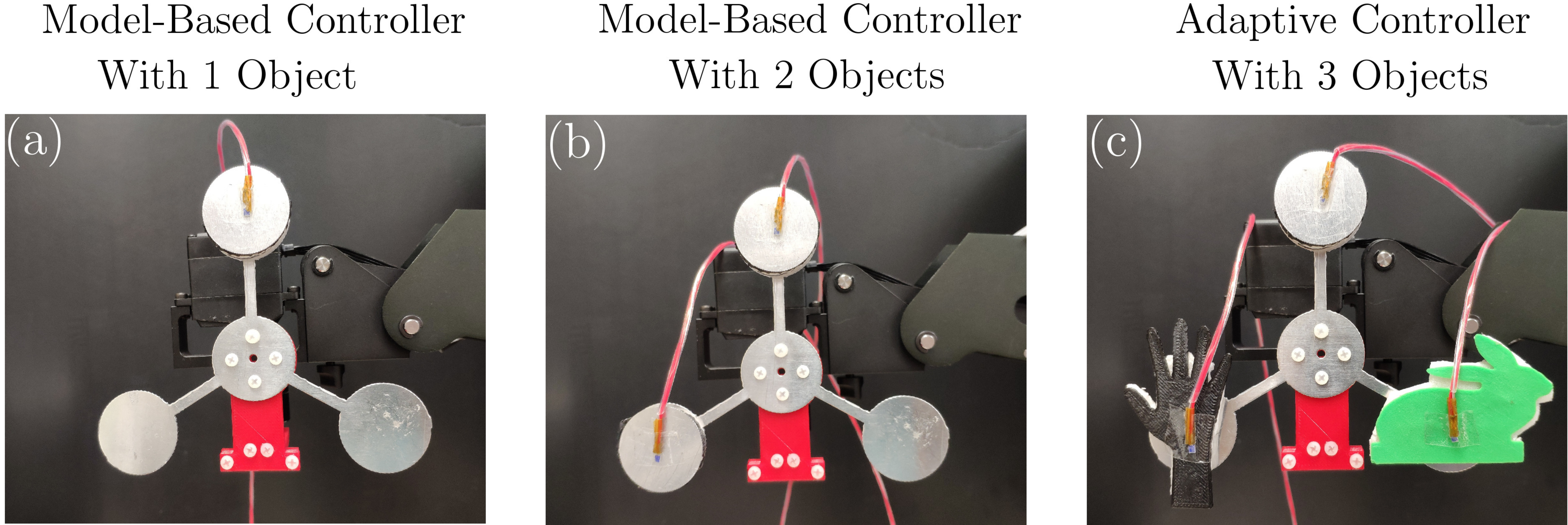}
				\caption{Different object used to test the two controllers.}	
				\label{fig:3-end-effector}
\end{figure}

\begin{figure}
				\centering
				\includegraphics[width =1 \columnwidth]{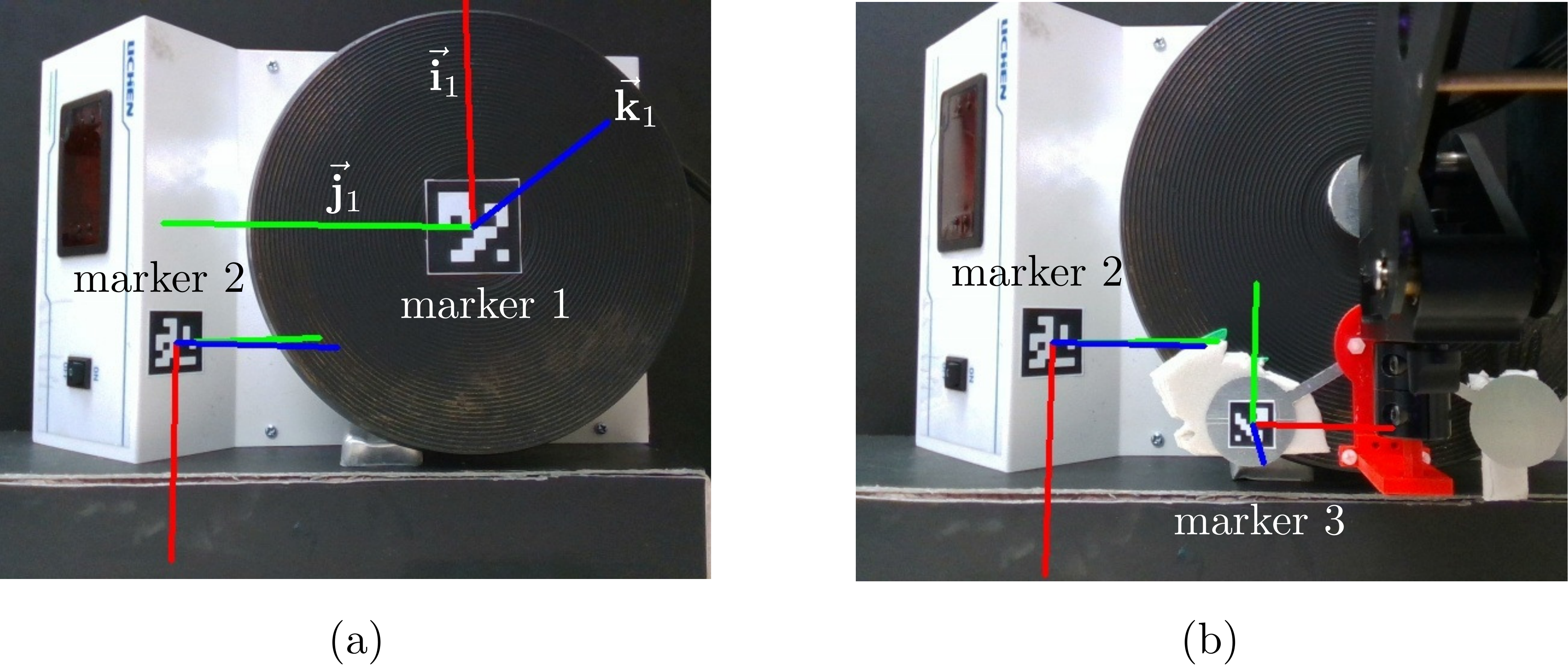}
				\caption{Calibration using ArUco markers before experiments.}
				\label{fig:calibration}
\end{figure}

We conducted a series of experiments on a 4-DOF robot (3 translations and 1 rotation) to evaluate the proposed method. 
Fig. \ref{fig: Set up} shows the robot, whose end-effector is replaced by a 3D printed connector fixed to an aluminum holder. 
The objects are attached to the holder through an adiabatic layer to minimize heat conduction. 
We prepared three different kinds of objects for temperature control experiments (see Fig. \ref{fig:3-end-effector}): An aluminum circular sheet with $\SI{1.5}{\cm}$ radius and $\SI{3}{\mm}$ thickness; A bunny-shaped object with $\SI{1}{\mm}$ thickness, 3D printed using polylactic acid (PLA) material with $30\%$ infill density; A hand-shaped sheet with $\SI{1}{\mm}$ thickness, also 3D printed using PLA but with $50\%$ infill density. 
We approximate the aluminum sheet's and the heat source's thermophysical properties via standard tables \cite{Fundamentals_of_heat_and_mass_transfer}. 
The object's emittance, absorptance, specific heat, and density are 0.04, 0.04, $\SI[inter-unit-product =\ensuremath{\cdot}]{903}{\joule\per\kelvin\per\kg}$, and $\SI[inter-unit-product =\ensuremath{\cdot}]{2702}{\kg\per\meter^{3}}$, respectively. 
The source's emittance and absorptance are estimated as 0.25 and 0.25.

For the two 3D printed objects, different infill densities, colors, and uncertain surface conditions make their thermophysical properties hard to be estimated. 
Thus, we only consider the aluminum sheet for the experiments with the model-based controller (adaptive control is used for the other objects).  
A radiative heating platform with adjustable temperature output is used as the heat source. 
The (indoor) environment temperature is assumed to be constant at $\SI{23}{\celsius}$.

To obtain the feedback temperatures, we attach a PT100 platinum thermistor with $\SI{0.3}{\celsius}$ accuracy and $\SI{0.1}{\celsius}$ precision to each object. 
The raw data obtained by thermistors is processed by a current-temperature transformation module and sent to a Linux-based control computer as the feedback signal. 
The motion command is calculated by the computer program and sent to the robot under a position-stepping mode.
At the beginning of the experiments, we use an RGB camera and three ArUco markers \cite{garrido2014automatic} to calibrate the configuration between the heat source and the end-effector (see Fig. \ref{fig:calibration}). 

\subsection{Experiments with the Model-Based Controller}
\begin{figure}
				\centering
				\includegraphics[width =1 \columnwidth]{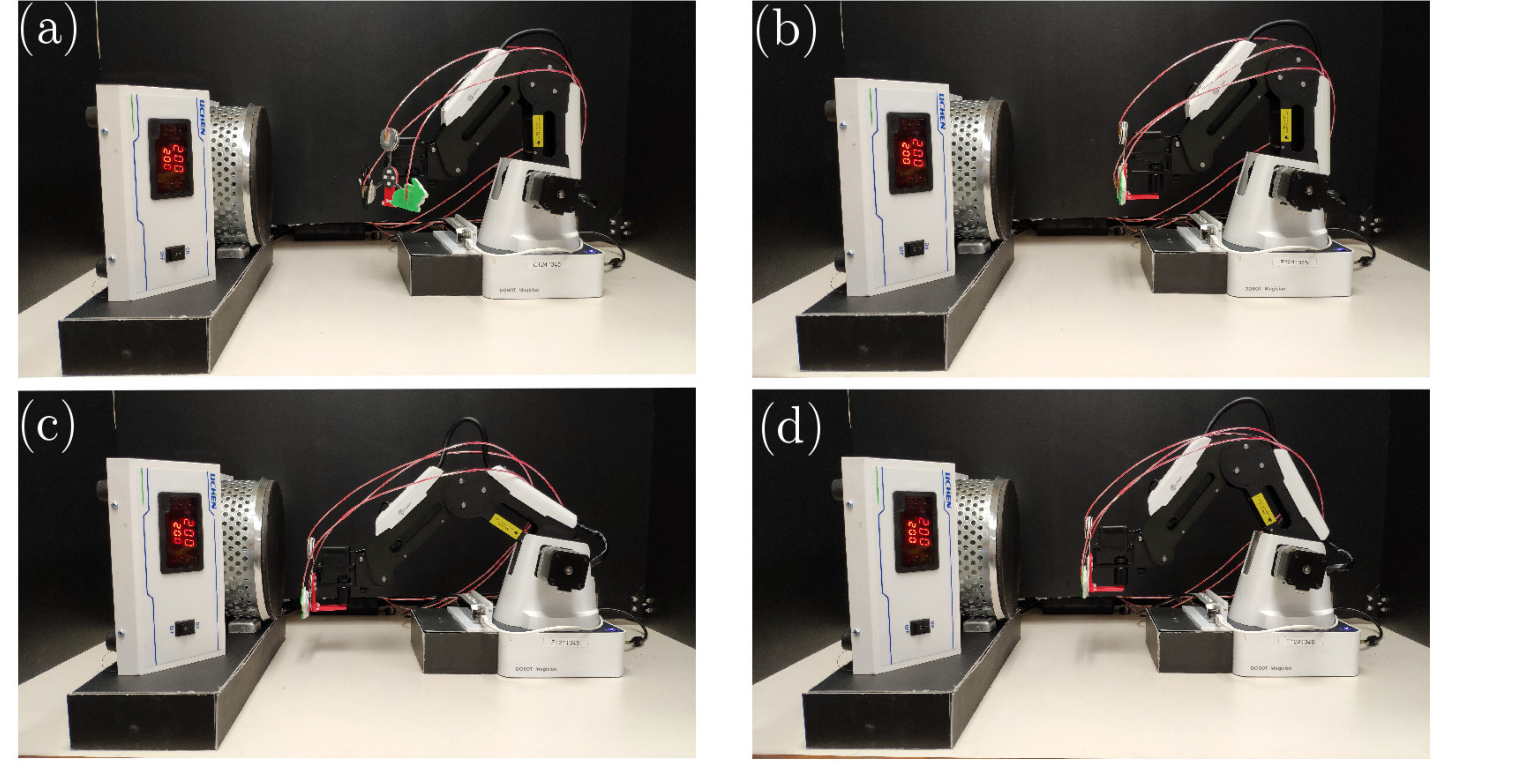}
				\caption{Snapshots of a representative thermal servoing experiment: (a) Initial position, (b)--(c) transient motion, and (d) steady-state configuration.}
				\label{fig:4-pic-porcess}
\end{figure}

\begin{figure}
				\centering
				\includegraphics[width =1 \columnwidth]{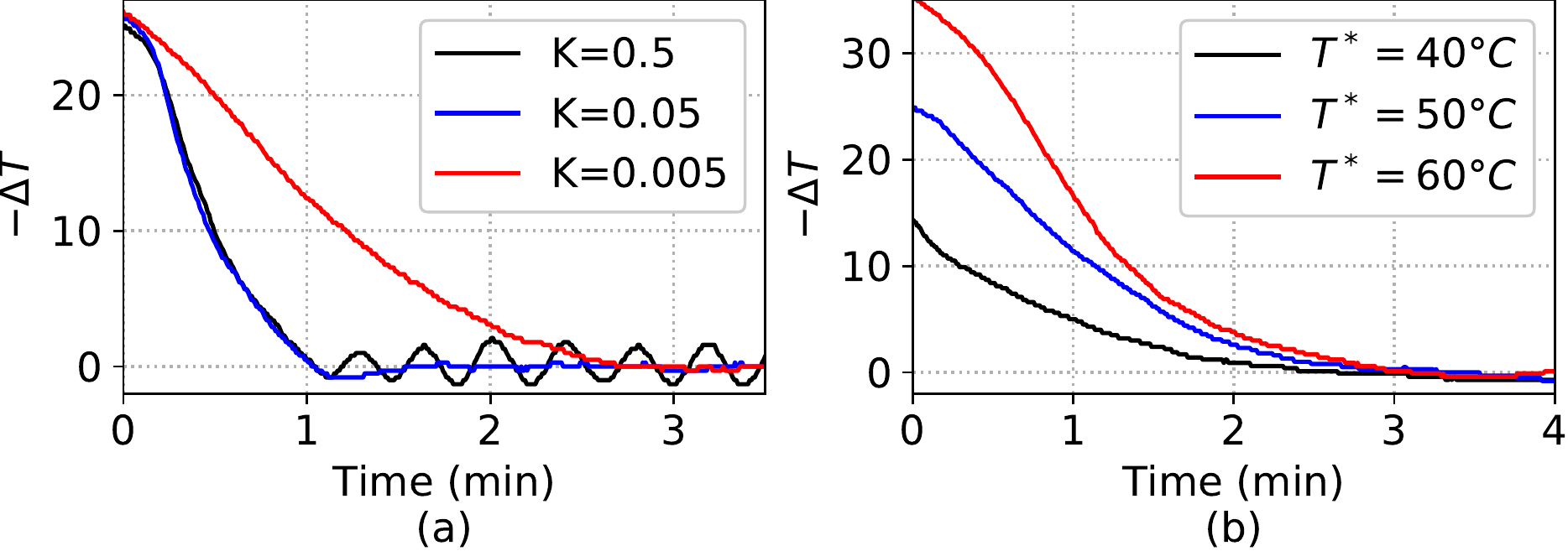}
				\caption{Evolution $\Delta T$ of the temperature error using one aluminum object with the model-based controller.}
				\label{fig:target-coefficient}
\end{figure}
\begin{figure}
				\centering
				\includegraphics[width = \columnwidth]{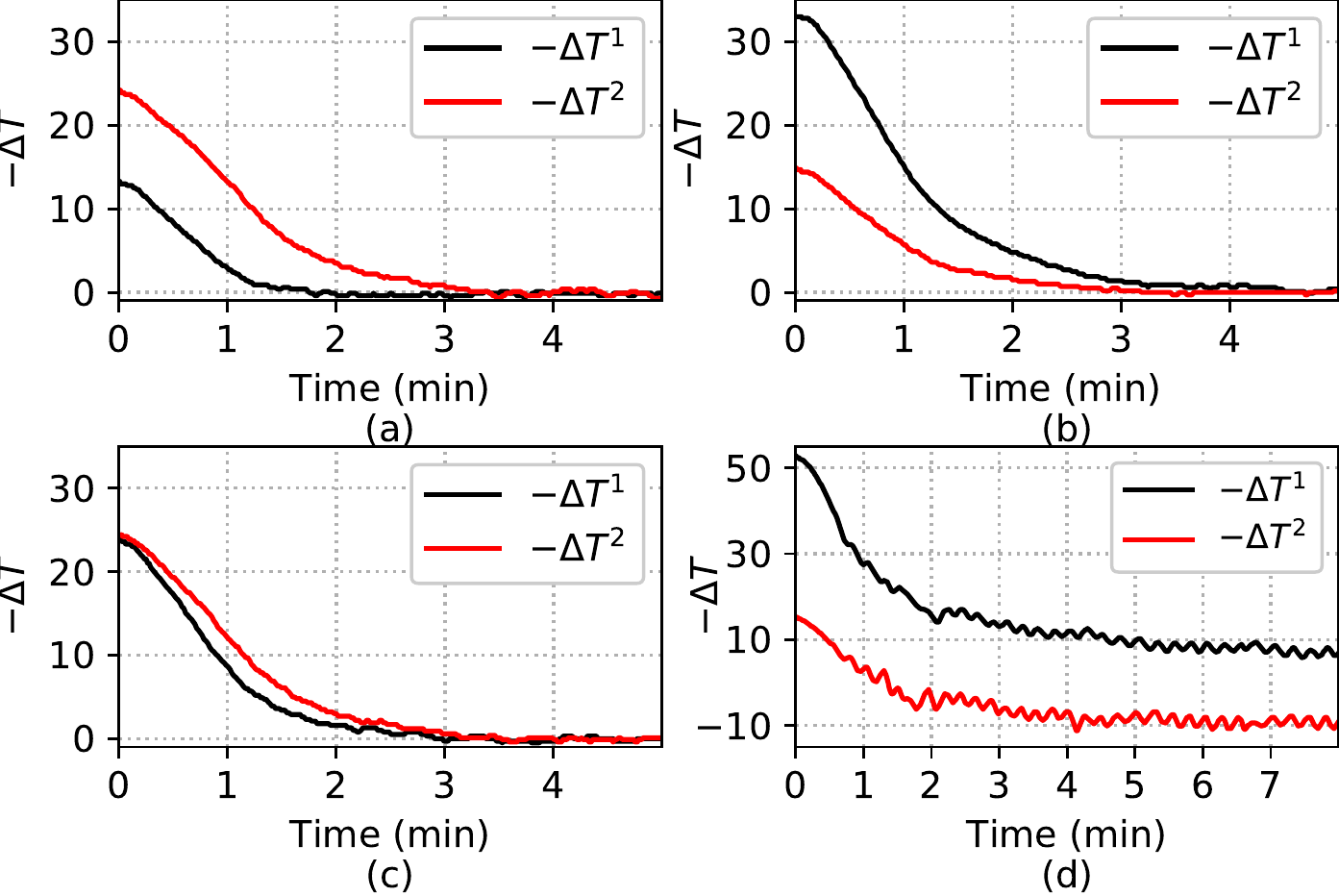}
				\caption{Evolution of the temperature errors ($\Delta T^{1}, \Delta T^{2}$) using two aluminum objects with the model-based controller.}
				\label{fig:two-objects}
\end{figure}

We conduct a series of thermal servoing experiments to evaluate our proposed control methodology (see Fig. \ref{fig:4-pic-porcess} for a representative experiment).
Here, we first evaluate the performance of the model-based controller with aluminum objects (whose properties are approximately known); The experiments are conducted with a source's temperature of \SI{200}{\celsius}. 

We denote the scalar temperature error by $\Delta T= T_{2}-T^{*}$. 
By using the controller \eqref{control input}, we enforce a closed-loop heat transfer system that resembles a mass-spring-damper system. 
Therefore, the values of the stiffness/damping-like gains $K$ and $D$ can be used to specify the system's performance. 
Fig. \ref{fig:target-coefficient} (a) demonstrates the effect of the gain $K$ on the thermal response. 
For that, we set $D=0.2$ and $T^{*}=\SI{50}{\celsius}$ and conduct three experiments with different $K$ values. 
These show that when $K=0.005$ (red curve), the error $\Delta T$ asymptotically decreases to zero with a relatively slow speed; When $K=0.05$ (blue curve), $\Delta T$ decreases faster and a small overshoot occurs; When $K= 0.5$ (black curve), $\Delta T$ oscillates near zero with an approximate $\SI{3}{\celsius}$ amplitude. 
This results shows how the closed-loop system varies from over-damped to under-damped. 
Thus, the gains should be specified according to the desired thermodynamic performance. 
We further conducted experiments with the same gains ($K=0.05$, $D=0.2$) but with different targets $T^*$ and found a consistent response (see Fig. $\ref{fig:target-coefficient}$ (b)). 
 
Model-based experiments were also conducted to independently regulate the temperatures of two aluminum objects, shown in Fig. \ref{fig:3-end-effector} (b). 
We designed 4 experiments with different targets $\boldsymbol{\tau}^*$ (measured in $\SI{}{\celsius}$). 
Fig. $\ref{fig:two-objects}$ depicts the minimization of the thermal errors for these 4 experiments, with target temperatures defined as $\boldsymbol{\tau}^{*} = [{50},{40}{]}^{\T}$, $\boldsymbol{\tau}^{*}=[{60},{40}{]}^{\T}$, $\boldsymbol{\tau}^{*} = [{50},{50}{]}^{\T}$, and $\boldsymbol{\tau}^{*} = [{80},{40}{]}^{\T}$ in (a), (b), (c), and (d), respectively. 
For the first three experiments where the differences between the target temperatures $|T^{*1}-T^{*2}|$ are small (or null), the thermal error $\|\Delta \boldsymbol{\tau}\|$ can be asymptotically minimized to zero.
However, when $|T^{*1}-T^{*2}|$ is large, as in Fig. \ref{fig:two-objects} (d), the two temperatures cannot be accurately controlled. 
This failure case can be explained by the second condition for feasible targets discussed in Section \ref{section:Target Feasibility}.

\subsection{Experiments with the Adaptive Controller}
We designed a series of experiments to evaluate the performance of the proposed adaptive controller. 
For that, we consider with three different objects (see Fig. \ref{fig:3-end-effector} (c)) with unknown thermophysical properties and irregular shapes.
To compute the interaction matrix, we use truncated Fourier series with 5 harmonics terms; This approach provides a fast calculation time with a ``good enough'' shape approximation.
The controller's gains are set to $\mu=0.05$ and $K=0.15$. %We use learning gains of $\gamma_{1}=0.005$,  $\gamma_{2}=10^{-22}$ to constrain $\dot{\widehat{\mathbf{a}}}_1$, $\dot{\widehat{\mathbf{a}}}_2$ to have the same order of magnitude in each iteration. 
To initialize the parameters $\widehat{\mathbf a}_i(0)$ at the time instance $t=0$, we use (for the ``hand'' and ``bunny'' objects) the constant values calculated for the aluminum object in the previous model-based controller, i.e. $\widehat{\mathbf a}_i(0)=\mathbf a_i$; For the circular object, we simply initialize $\widehat{\mathbf a}_i(0)$ with random values.

In this study, we report eight temperature control experiments with different targets, objects and source conditions. 
Figure \ref{fig:3 objects figure} shows the evolution of the individual thermal errors $\Delta T^i$. 
For ease of presentation, we name these eight experiments as $\textit{exp 1, \ldots, exp 8}$, and denote the corresponding target temperature for each experiment by $\boldsymbol{\tau}^{*1}, \ldots,\boldsymbol{\tau}^{*8}$. 
In \textit{exp 1}, we set the three target temperatures to the same value. 
In \textit{exp 2} -- \textit{exp 4}, we only set two targets to the same value. 
In \textit{exp 5} and \textit{exp 6}, we set all targets to different values, with a non-uniform thermal separated in \textit{exp 5} and a uniform one in \textit{exp 6}. 
In \textit{exp 7} and \textit{exp 8}, all targets are set to the same value, but with different heat source conditions. 
The source temperature $T_{1}$ is set to $\SI{200}{\celsius}$ in $\textit{exp 1}$ -- \textit{exp 6}, to $\SI{300}{\celsius}$ in $\textit{exp 7}$, and varies from $\SI{200}{\celsius}$ to $\SI{300}{\celsius}$ in $\textit{exp 8}$. 

In all these experiments with all these different conditions, the magnitude of the temperature error $\|\boldsymbol{\tau}\|$ asymptotically decreases to zero. 
Yet, failure control experiments do happen and are reported and discussed in Section \ref{sec: Result Unfeasible Thermal Targets}). 
The results experimentally confirm that (for \emph{feasible} target temperatures) the adaptive method is able to independently regulate temperatures of various objects with different shapes and materials, without exact knowledge of their thermophysical properties or the source's/environment's temperatures. 

Fig. \ref{fig:3 objects path} depicts the performed object trajectories during the experiments in Fig. \ref{fig:3 objects figure}.
The boundary of the circular heat source is depicted as a black circle (and ellipse). 
The color of a trajectory point represents the feedback temperature at that position; Variation from blue to red corresponds to a change from ``low'' to ``high''. 
For clarity, we depict two sets of trajectory visualizations from different viewing angles: For Fig. \ref{fig:3 objects path} $(a_{1}), (b_{1}),\ldots, (h_{1})$, the trajectories are viewed in $-\vec{\mathbf{k}}_{1}$ direction; For Fig. \ref{fig:3 objects path} $(a_{2}), (b_{2}),\ldots, (h_{2})$, the trajectories are viewed in $\vec{\mathbf{i}}_{1}$ direction. 

From these trajectory visualizations, we can see that when target temperatures are set to different values, the object with a higher target temperature usually reaches a position that is closer to the center of the heat source; This situation will be further discussed in the Section \ref{sec: visualization}. 
For the case when target temperatures are set to the same value, the final position of the circular aluminum sheet is always closer to the center of the heat source. 
This phenomenon could be explained by the fact that absorptance of a metal is usually much smaller than the absorptance of non-metallic materials (e.g. PLA) \cite{lienhard2005heat}. 

\subsection{Experiments with a Moving Heat Source}
\label{section: experiments with manual interference}

\begin{figure*} 
				\centering
				\includegraphics[width = 2 \columnwidth]{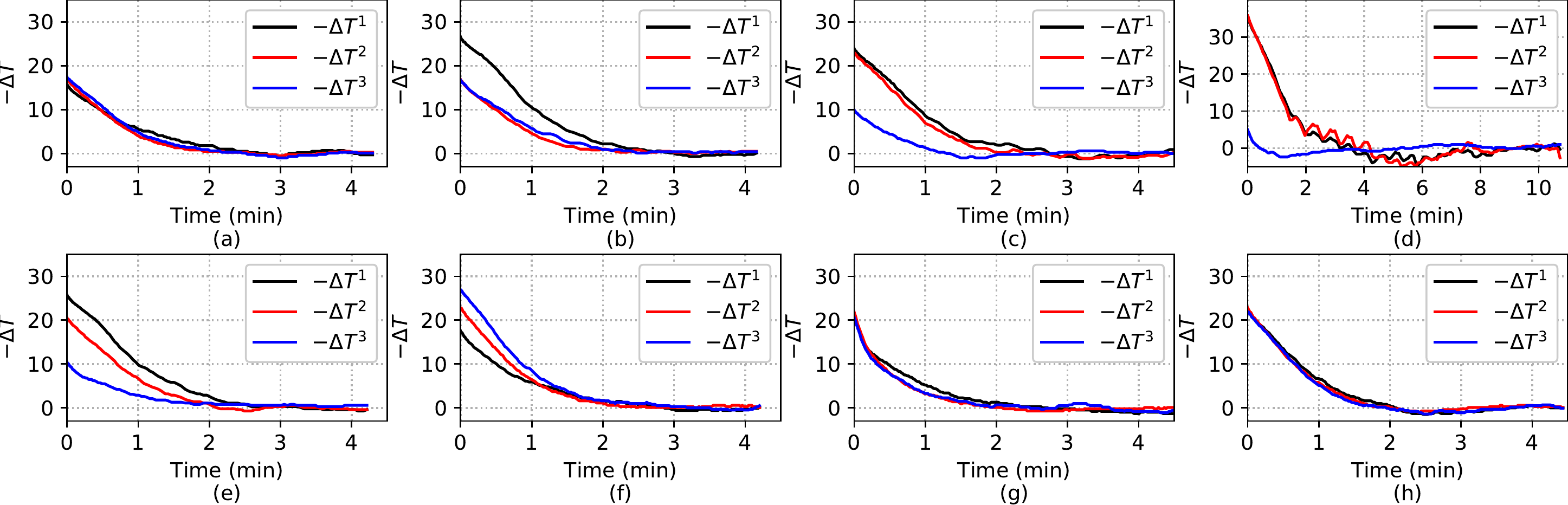}
				\caption{Evolution of the temperature errors of the three objects in the $8$ experiments with the adaptive controller,  $\Delta T^{1}$, $\Delta T^{2}$, and $\Delta T^{3}$ (measured in $\SI{}{\celsius}$). The target  temperatures are set as: $\boldsymbol{\tau}^{*1}=\begin{bmatrix}40&40&40\end{bmatrix}^{\T}$,  $\boldsymbol{\tau}^{*2}=\begin{bmatrix}50&40&40\end{bmatrix}^{\T}$,  $\boldsymbol{\tau}^{*3}=\begin{bmatrix}50&50&35\end{bmatrix}^{\T}$,  $\boldsymbol{\tau}^{*4}=\begin{bmatrix}60&60&30\end{bmatrix}^{\T}$,  $\boldsymbol{\tau}^{*5}=\begin{bmatrix}50&45&35\end{bmatrix}^{\T}$,  $\boldsymbol{\tau}^{*6}=\begin{bmatrix}40&45&50\end{bmatrix}^{\T}$,  $\boldsymbol{\tau}^{*7}=\begin{bmatrix}45&45&45\end{bmatrix}^{\T}$,  $\boldsymbol{\tau}^{*8}=\begin{bmatrix}45&45&45\end{bmatrix}^{\T}$.} 
				\label{fig:3 objects figure}
\end{figure*}
\begin{figure*}
				\centering
				\includegraphics[width = 2 \columnwidth]{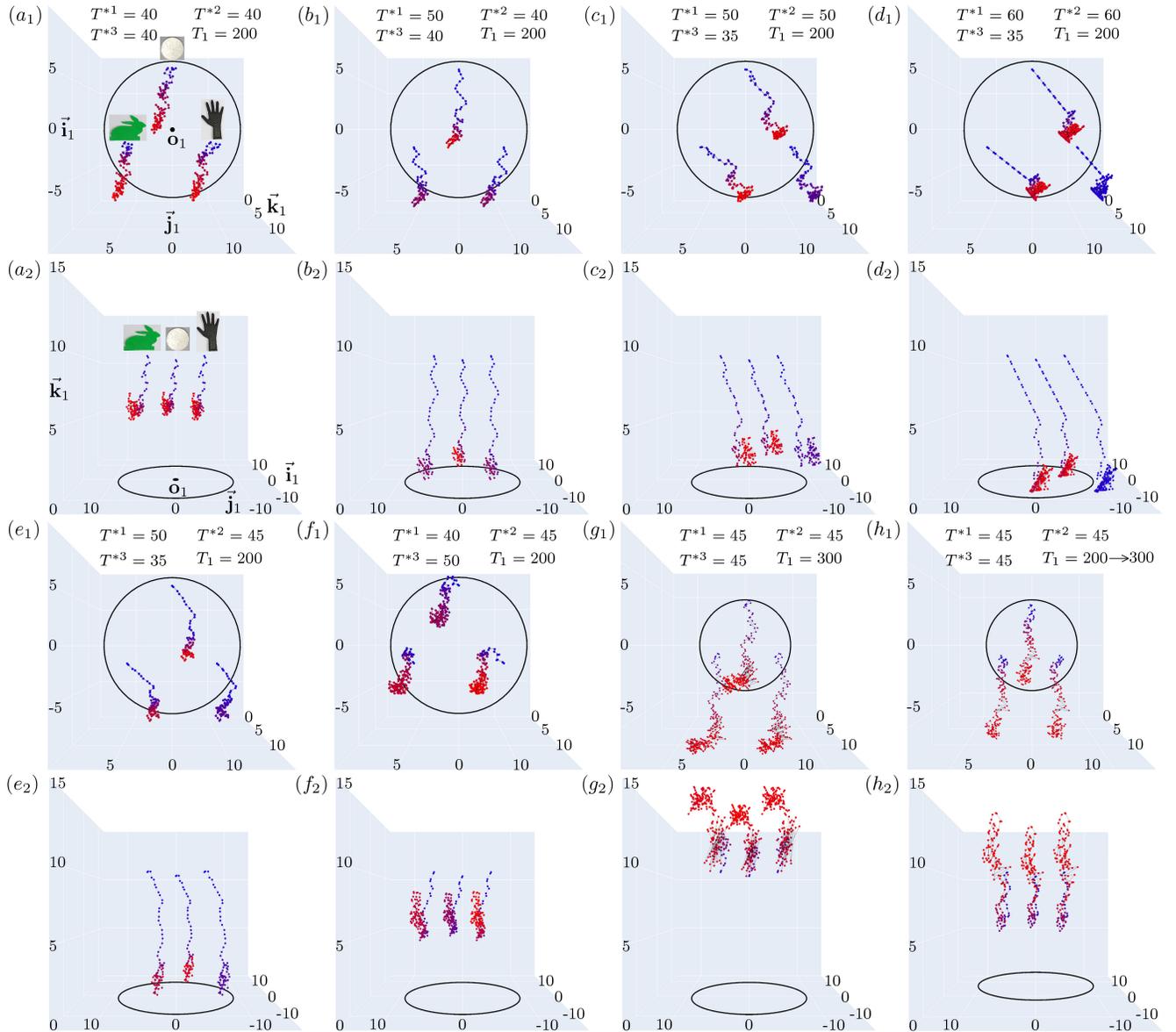}
				\caption{Spatial displacements of the three objects in the $8$ experiments with the adaptive controller visualized from two viewing angles. We use a blue-to-red color gradient to visualize the cold-to-hot change of temperatures during the experiments.}
				\label{fig:3 objects path}
\end{figure*}

In this section, we report an integrated experiment where the adaptive controller is combined with an online ArUco tracking algorithm to achieve temperature regulation while the heat source is changing. 
During setup, the ArUco markers attached to the heat source are used to geometrically calibrate the robot with the source, which is essential to compute the thermal interaction matrix $\mathbf L$.
Thus, when these relation are uncertain, their new configurations have to be updated simultaneously. 

\begin{figure}[t]
				\centering
				\includegraphics[width = \columnwidth]{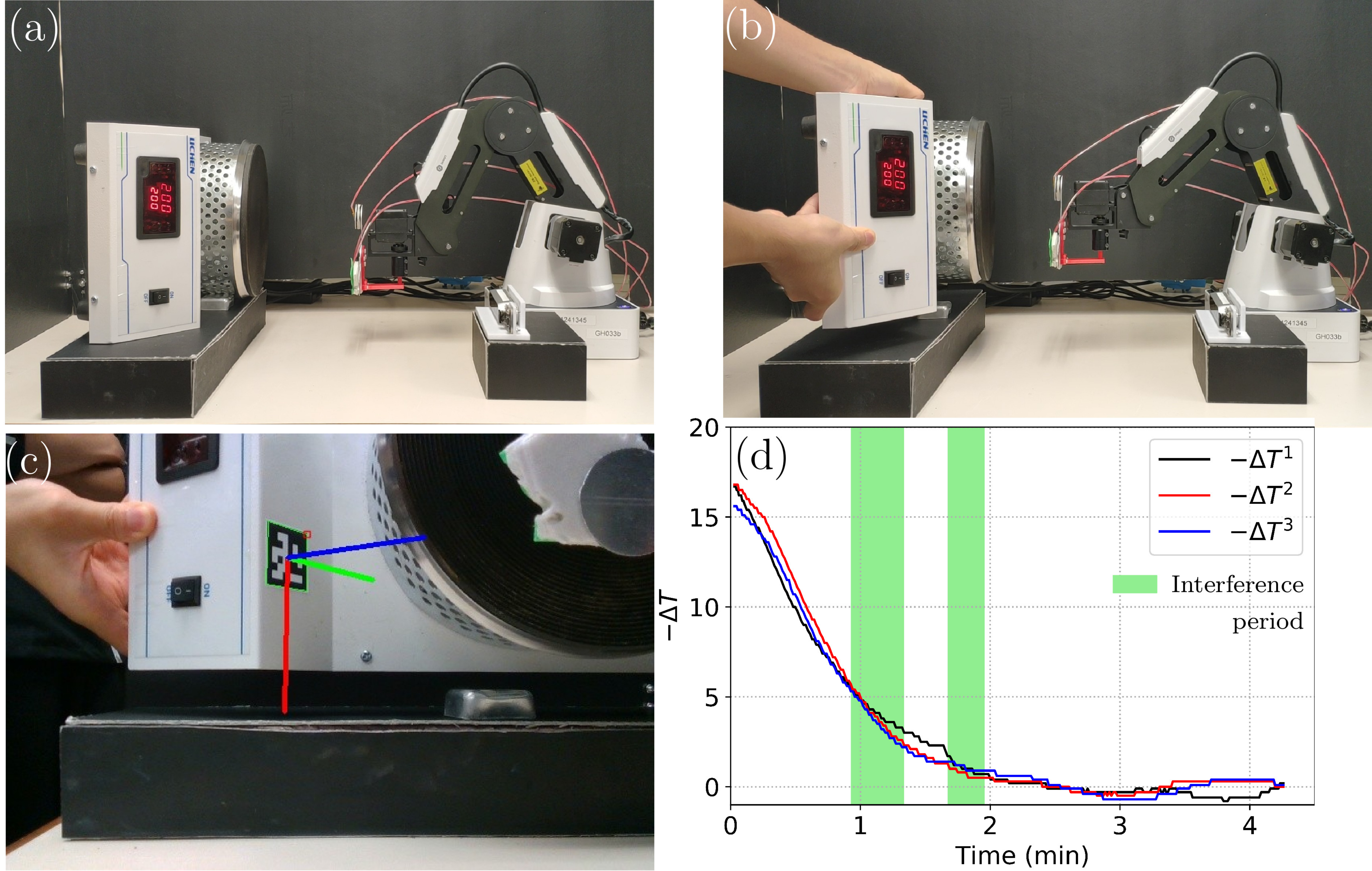}
				\caption{Experiments with the adaptive controller with disturbances.}
				\label{fig:interference experiments}
\end{figure}

Here, we study the case where the robot is fixed and the heat source is manually moved. 
We track marker $2$ attached to the source to obtain its configuration. 
The target temperature vector is set to $\boldsymbol{\tau}=[40,40,40{]}^\T$ $\SI{}{\celsius}$. 
Fig. \ref{fig:interference experiments} shows: (a) The initially calibrated set up, (b) the manual movement applied to the source, (c) the detected marker $2$ when the source is moving, and (d) the evolution of the individual temperature errors. 
These results show that by continuously updating the source-object pose, the control of the individual temperature errors is not significantly affected and that $\|\Delta\boldsymbol{\tau}\|$ can still be asymptotically minimized. 
This experiment demonstrates how our new thermal servoing method can be combined with other traditional controllers (visual servoing in this case) to extend the sensorimotor capabilities of a robot \cite{dna2020_front_neuro}.

\subsection{View Factor Visualization}

\label{sec: visualization}
In previous sections, we designed thermal controllers based on derived heat transfer models.
However, the models that relate $F_{21}$ and $\mathbf{x}$ are generally complex. 
Therefore, part of the controlled system behaves as a ``black box'' to the user. 
To investigate these aspects, in this section, we introduce the visualization of the view factor $F_{21}$ with respect to the end-effector configuration $\mathbf{x}$ as a useful tool for analyzing radiation-based thermal servoing problems.

As an example, we take the ``circular surfaces in arbitrary configurations'' case discussed in Section \ref{section:Circular Surfaces at Arbitrary Configurations}.
We implement the controlled variable method \cite{williams2007research} to split the 6-DOF pose $\mathbf{x}$ into two subsets: One where the translation coordinates $p_{1}$, $p_{2}$, and $p_{3}$ (measured in \SI{}{\cm}) are the controlled variables, and another where the rotation coordinates $\theta_{x}$, $\theta_{y}$, $\theta_{z}$ (measured in degrees) are controlled variables. 
In the translation subset, rotations are set to constant values of $\theta_{i}=0$, then, we compute $F_{21}$ for points in a selected working range of controlled variables $p_{1}, p_{2} \in [-20, 20]$, $p_{3} \in [0, 30]$ with a step of $1$ (with $48,000$ points in total). 
In the rotation subset, translations are similarly set to fixed constant values $p_{1,2}=0$ and $p_3=\SI{5}{\cm}$, then, points in the range of $\theta_{x},\theta_{y},\theta_{z} \in [-90,90]$ are computed with an incremental step of $2$ (with $729,000$ points in total).

We use the isosurface visualization tool provided by $\textit{Plotly}$ to visualize the data. 
The translation and rotations subsets are shown in Fig. \ref{fig:isosurface}, where 3-DOF end-effector configurations are represented by points in space, and the view factor values are represented by isosurfaces with different colors (the isosurfaces are formed by points which have the same or very close values of $F_{21}$). 
This visualization method is inspired by the approximation of the interaction matrix in \eqref{nummerical approximation of L_inv}, which reveals that $\mathbf L$ is positive proportional to the directional derivative of $F_{21}(\mathbf{x})$ along $\mathbf{x}$ as $\mathbf{L}^\T=\lambda_{1} \nabla_{\mathbf{x}} F_{21}(\mathbf{x})$. 
According to the definition of isosurface, the surface normal of every point on the surface also points in $\nabla_{\mathbf{x}} F_{21}(\mathbf{x})$ direction. 
In addition, the interval distance between isosurfaces with an equal value difference (also called ``isosurface interval'') reveals the magnitudes of the elements of $\nabla_{\mathbf{x}} F_{21}(\mathbf{x})$; A larger distance represents a smaller magnitude. 

As an example, let us analyze the translation subset shown in Fig. \ref{fig:isosurface} (a). 
For this single-object scenario, the normal vector at a point on the isosurface indicates the direction of the end-effector movement (as computed from the thermal controls \eqref{control input}, \eqref{eq:adaptive control input}) at that point. 
There are some characteristics of these isosurfaces that can be intuitively deducted from the setup, e.g., the symmetric spatial distribution of $F_{21}$ (due to the circular shape of the heat source), and the proportionality of values of $F_{21}$ with respect to the source-object separation. 

However, the visualization provides two useful pieces of information. 
First, that the centers of the incomplete spherical isosurfaces shift upwards when $F_{21}$ decreases, which means that at some points, movement in the $\vec{\mathbf{k}}_{1}$ direction will cause a decrease of $F_{21}$ (which seems counter-intuitive). 
See e.g. $\mathbf{c}_{2}$ on the $F_{21}=0.1$ isosurface in Fig. \ref{fig:isosurface} (a), which shows that in that configuration, the end-effector needs to move backwards along the $\widehat{\mathbf{k}}_{1}$ direction to heat up faster. 
Second, the isosurface intervals at regions that are farther from the heat source center are comparatively larger, which indicates that the end-effector will move comparatively faster in those regions.
Similarly, Fig. \ref{fig:isosurface} (b) shows the (much simpler) case where angles are varied at a fixed position.
 
To further analyze the effect of robot motion on heat transfer process, we now present a 4-DOF visualization result in Fig. \ref{fig: 9_view_factors}. 
Based on the previous 3-DOF translation subset, we add one more controlled variable $\theta_{x}$ to extend it to 4-DOF.
We vary the rotation $\theta_{x}$ from $0$ to $90 \SI{}{\celsius}$ and depict the change of translation subset isosurfaces in Fig. \ref{fig: 9_view_factors} (a), (b), (c).
For the previous rotation subset, we add one more controlled variable $p_{3}$ and vary it from $1$ to $10$ \SI{}{\cm}, as shown in Fig. \ref{fig: 9_view_factors} (d), (e), (f).
Similarly, we vary $p_{1}$ from $0$ to $10$ \SI{}{\cm} and depict the change of rotation subset isosurfaces in Fig. \ref{fig: 9_view_factors} (g), (h), (i).
An animation of this experiment is included in the supplementary video.

\begin{figure}[t]
                \centering
                \includegraphics[width = 1\columnwidth]{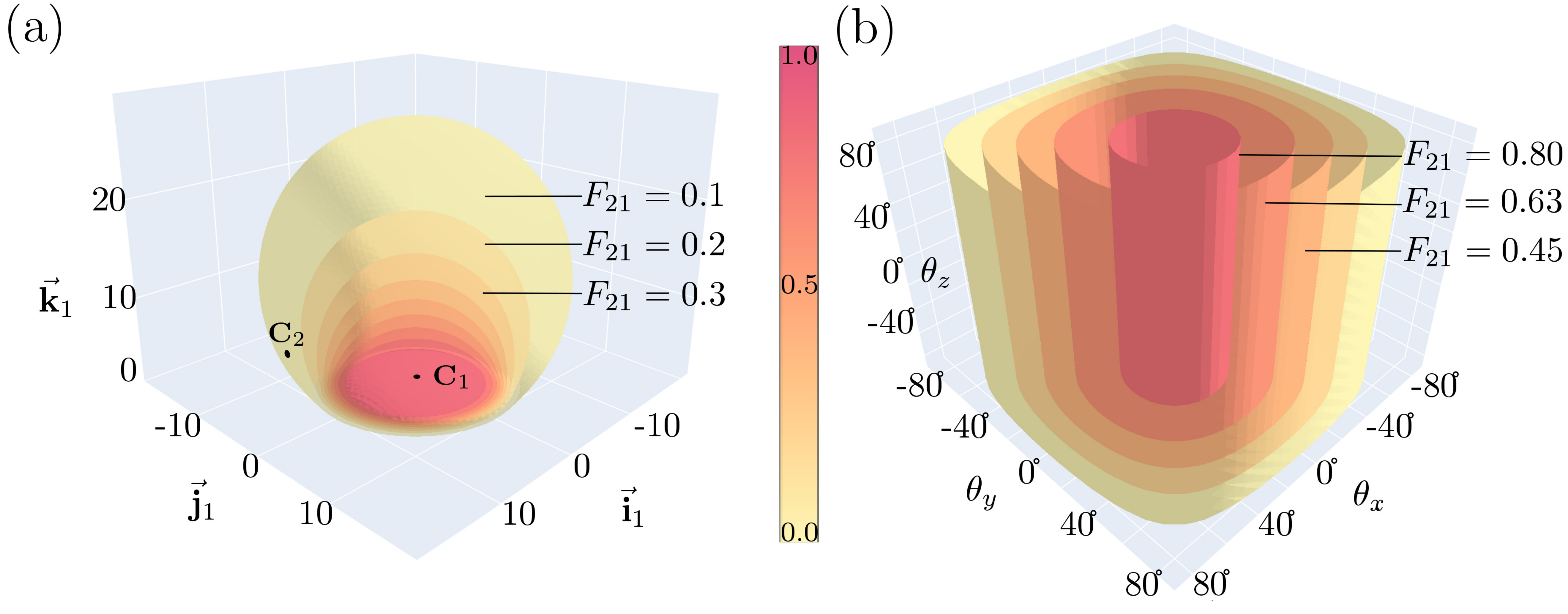}
                \caption{Isosurfaces visualization of the two view factor subsets: (a) Translation subset and (b) rotation subset.}
                \label{fig:isosurface}
\end{figure}
\begin{figure}
                \includegraphics[width =1 \columnwidth]{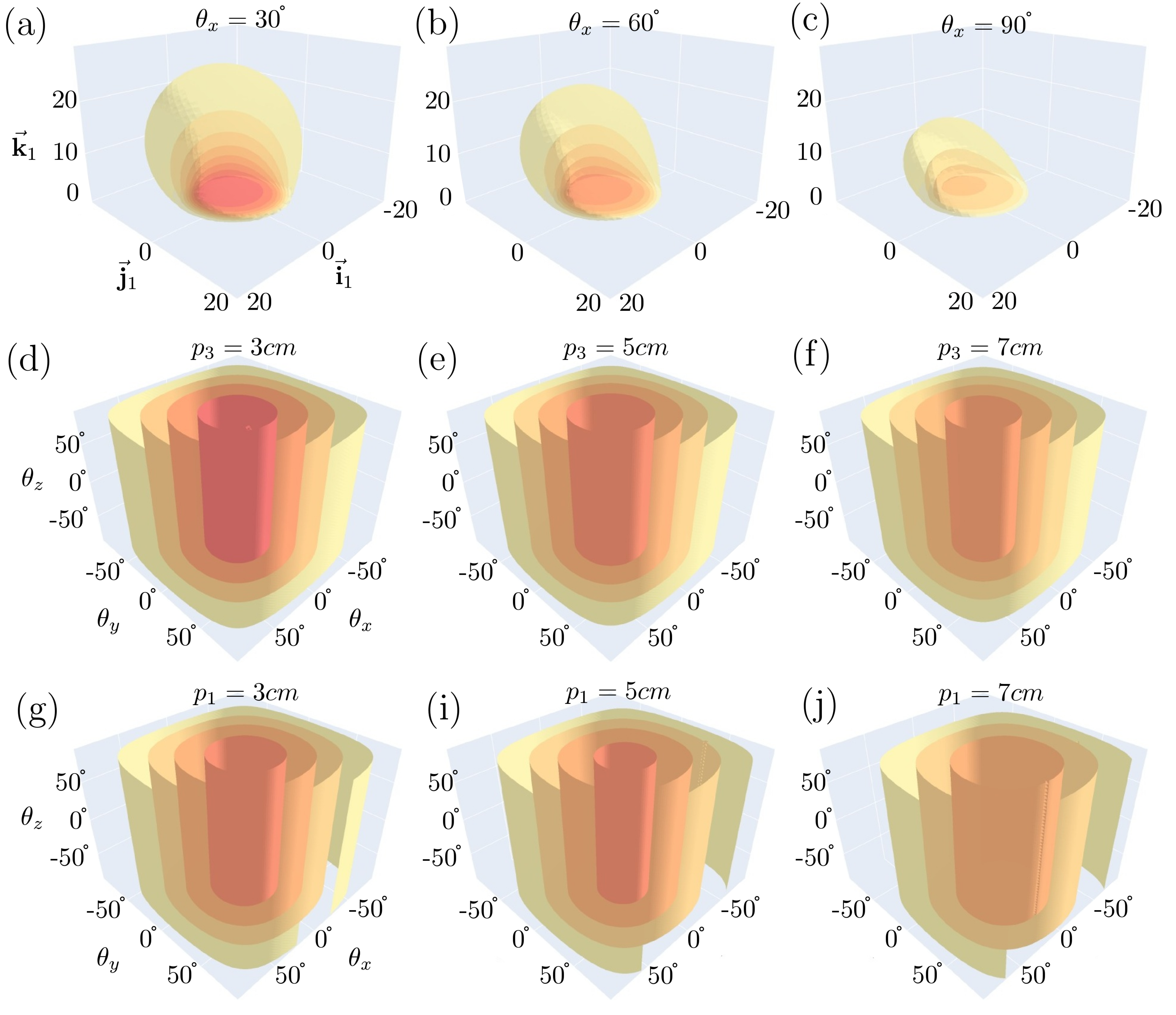}
                \caption{Controlled variable visualization of the view factor isosurfaces.}
                \label{fig: 9_view_factors}
\end{figure}

\subsection{Unfeasible Thermal Targets}
\label{sec: Result Unfeasible Thermal Targets}

In Section \ref{section:Target Feasibility}, we discussed two necessary but not sufficient conditions for feasible targets. 
When one of the two conditions is not fulfilled, the temperature error cannot be minimized to zero. 
This section reports and analyzes two failed experiments with the proposed adaptive controller where the target temperatures are set to $\boldsymbol{\tau}^{*} = [80,80,80{]}^\T$ and $\boldsymbol{\tau}^{*} = [70,35,35{]}^\T$ \SI{}{\celsius}; The temperature errors of each coordinate $\Delta T^{i}$ are depicted in Fig. \ref{fig:unfeasible targets} (a), (c). 
The evolution of the error $\|\Delta \boldsymbol{\tau}\|$ for the two experiments is shown in Fig. $\ref{fig:unfeasible targets}$ (b), (d). 
In this experimental study, we found that when all the individual target temperatures are set relatively ``high'' ($80$ \SI{}{\celsius} in this example), its corresponding errors converge to a local minimum.
Also, when the difference between individual target temperatures is too large, one of the objects might more closely reach its target, while the other will present steady-state errors. 

In Section \ref{section:Target Feasibility}, we prove that the steady-state temperature of an object heated by a radiative source is directly proportional to $F_{21}$. 
Thus, the geometry of view factor isosurfaces is a useful tool for analyzing the such reachability conditions. 
Here, we discuss a simple but representative case where two aluminum circular sheets with radius $r_{o1}=\SI{1.5}{\cm}$, $r_{o2}=\SI{4.5}{\cm}$ are attached to the end-effector at $\mathbf{o}_{1}$ and $\mathbf{o}_{2}$ (see Fig. $\ref{fig: unfeasible_end_effector}$) and heated by a source with $T_1=\SI{200}{\celsius}$. 
The center of the end-effector is at $\mathbf{o}_{e}$, and $l_{e1}=l_{e2}=\SI{2}{\cm}$ are the distances between the centers $\mathbf{o}_{1}$ and $\mathbf{o}_{2}$ and the end-effector $\mathbf{o}_{e}$. 
The view factors of the objects are calculated based on the same setup as in previous sections. 
We use the visualization method where three translations are the controlled variables, for a parallel object and source surfaces. 

By using the expression \eqref{eq:steady state temperature} and assuming that the thermophysical properties are the same as mentioned in Section \ref{sec:set up}, the view factor values corresponding to steady-state temperatures $\SI{30}{\celsius},\SI{40}{\celsius},\SI{50}{\celsius}$ are calculated as $0.12$, $0.37$, and $0.65$. 
According to this one-to-one correspondence between the isosurface and the steady-state temperature, to automatically reach the target temperature $T^{*}$ can be geometrically interpreted as positioning the object center over the isosurface that corresponds to $T^{*}$.
Similarly, determining the feasibility of target temperatures $T^{*1}$ and $T^{*2}$ of two objects attached to the same end-effector is identical to finding whether there exists an end-effector pose that places both objects onto their ``desired isosurfaces''.

An example is shown in Fig. \ref{fig: unfeasible_isosurfaces}, where we denote the steady-state temperatures of objects 1 and 2 by $T_{ss}^{1}$ and $T_{ss}^{2}$. 
Fig. \ref{fig: unfeasible_isosurfaces} (a)--(b) show the steady-state temperature isosurfaces of objects 1 and 2 where $T_{ss}^{1}=T_{ss}^{12}=\SI{30}{\celsius},\SI{40}{\celsius},\SI{50}{\celsius}$. 
Since the two objects are circular plates with different radii, the shapes of their isosurfaces are slightly different; We use red and blue color to differentiate them, and are jointly depicted in the same coordinate system in Fig. \ref{fig: unfeasible_isosurfaces} (c). 
 
\begin{figure}[t]
                \centering
                \includegraphics[width =1 \columnwidth]{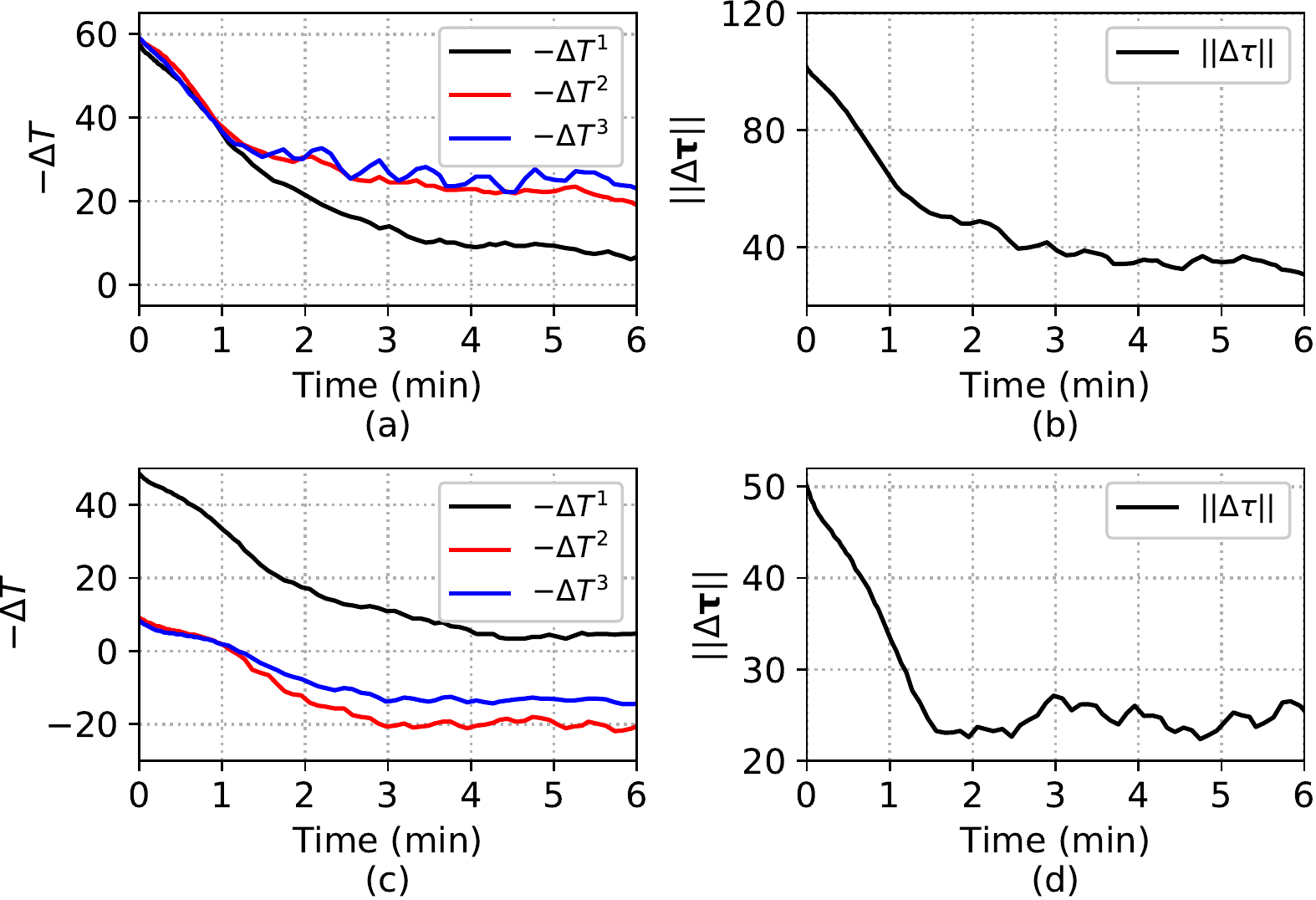}
                \caption{Evolution of the temperature errors with unfeasible target temperatures: (top) $\boldsymbol{\tau}^{*} = [80,80,80{]}^\T$ \SI{}{\celsius}, and (bottom) $\boldsymbol{\tau}^{*} = [70,35,35{]}^\T$.}
                \label{fig:unfeasible targets}
\end{figure}

In Fig. \ref{fig: unfeasible_isosurfaces} (d), (e), (f), different combinations of target temperatures $T^{1*},T^{2*}$ and their corresponding isosurfaces are depicted. 
These figures graphically demonstrate how for thermal targets $T^{*1}=T^{*2}=\SI{30}{\celsius}$ (depicted in Fig. \ref{fig: unfeasible_isosurfaces} (d)), and $T^{*1}=\SI{30}{\celsius}$, $T^{*2}=\SI{40}{\celsius}$ (depicted in Fig. \ref{fig: unfeasible_isosurfaces} (e)), the end-effector can position the objects into their desired final isosurfaces (corresponding to their target steady-state temperatures); The initial position of this trajectory is colored in blue, the final in red.
However, for the case where $T^{*1}=\SI{50}{\celsius}$ and $T^{*2}=\SI{30}{\celsius}$, Fig. \ref{fig: unfeasible_isosurfaces} (f) graphically demonstrates that the minimum distance between the two isosurfaces is larger than $l_{e1}+l_{e2}$; Therefore, these target temperatures $T^{*1}$ and $T^{*2}$ are \emph{not feasible}. 
Similarly, if $l_{e1}+l_{e2}$ is larger than the maximum distance between two target isosurfaces, that combination of $T^{*1}$ and $T^{*2}$ is also unfeasible. 
 
In general, thermophysical properties, view factors, and fixed spatial relationships between objects are the main three factors that determine the feasibility of thermal targets. 
The geometric interpretation of the feasibility problem might also be useful for path planning-like algorithms dealing with thermal servoing problems.

The accompanying multimedia file demonstrates the performance of our new control methodology with multiple experimental results. 

\begin{figure}[t]
                \centering
                \includegraphics[width =1 \columnwidth]{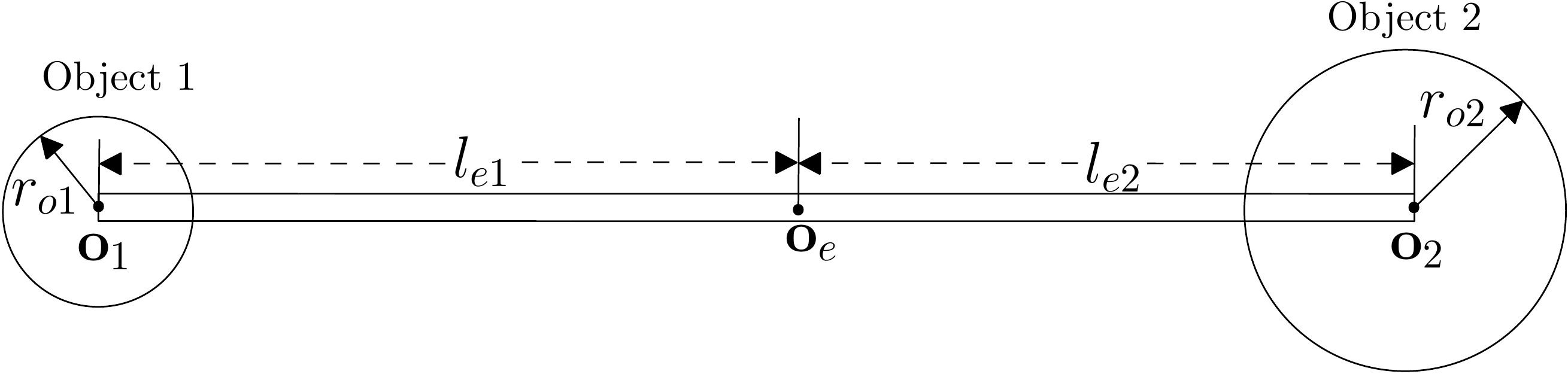}
                \caption{Conceptual illustration of two objects fixed to an end-effector for analyzing unfeasible target temperatures.}
                \label{fig: unfeasible_end_effector}
\end{figure}

\begin{figure}
                \centering
                \includegraphics[width =1 \columnwidth]{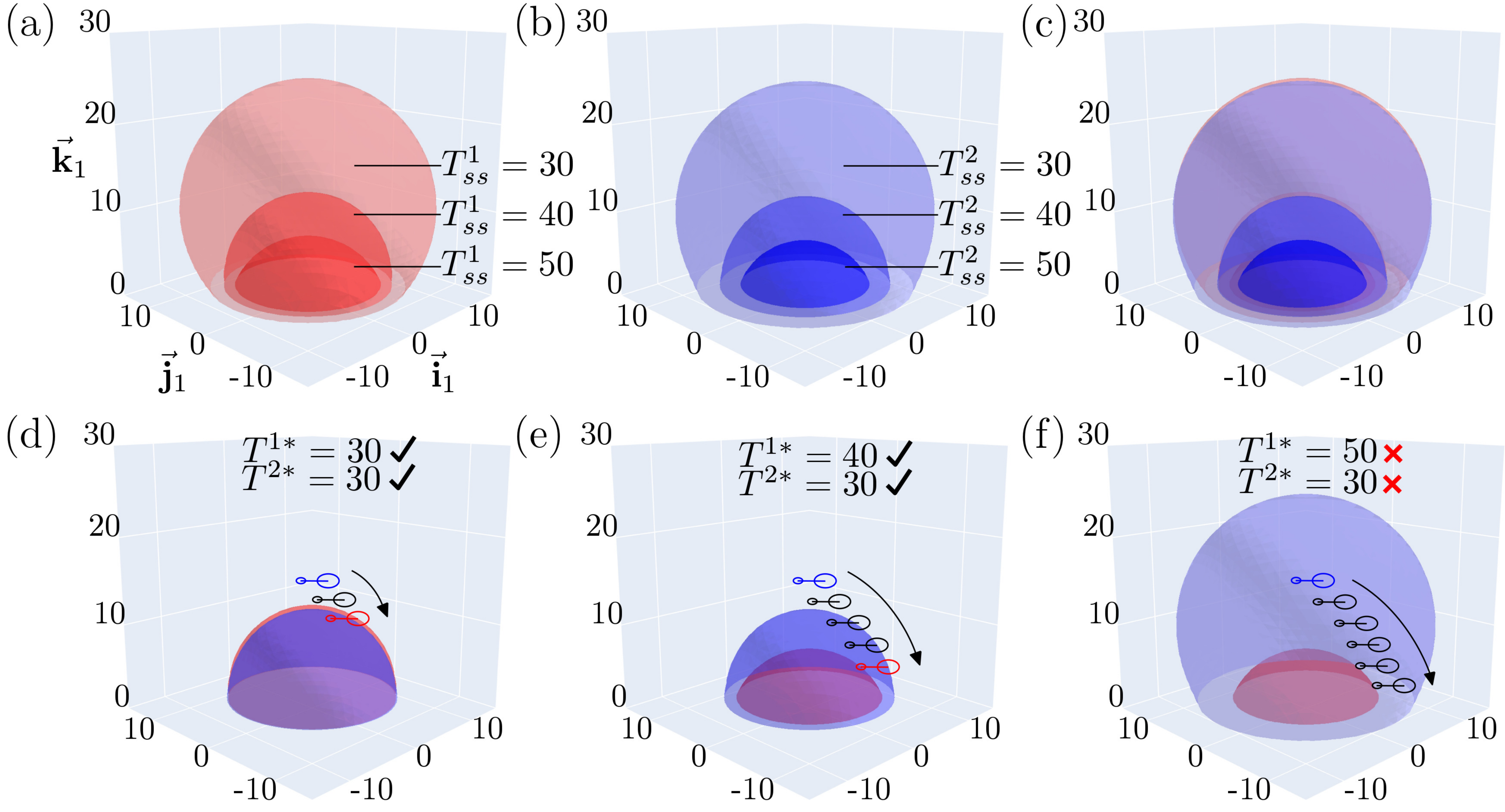}
                \caption{Geometric explanation of the target temperature feasibility using the steady-state temperature isosurfaces.}
                \label{fig: unfeasible_isosurfaces}
\end{figure}

\section{CONCLUSION}
\label{conclusion}
In this paper, we present a new robotic temperature control technique based on heat radiation to automatically regulate temperatures of multiple objects. 
For that, we provide a comprehensive formulation of different scenarios of thermal servoing problems.
Two asymptotically stable controllers, one model-based and one adaptive, are designed and validated by a series of experiments where temperatures of three different objects are independently regulated. 
We also discussed potential applications of the isosurface visualization, such as analyzing the geometry of the seemingly invisible heat transfer process.

The key concept of the proposed method is to exploit the geometric-thermal-motor relations between the heat source and the surface for automatically computing motion controls. 
This advanced feedback control capability is needed to improve the performance of many economically-important applications. 
However, from the point of view of generality, the proposed algorithm has many limitations, since we only consider cases where heat radiation is dominant. 
For cases where other heat transfer modes are dominant or comparable (objects in contact with non-adiabatic surfaces, electrical equipment cooled by high speed air flow, human skin treated by laser thermal excitation, food heated up in a pan, etc.), different heat transfer models need to be analyzed. 
Another possible solution is to implement model-free control algorithms that primarily rely on collected data instead of analytical models (which in our case were formulated based on fundamental physical principles).

For future work, we would like to integrate thermal servoing with existing visual and proximity servoing algorithms; This multimodal perceptual and control capability is essential for developing advanced robotic temperature control systems in complex scenarios, such as service tasks in human environments and and intelligent industrial manufacturing. 
Our team is currently working towards developing algorithms which simultaneously consider the three basic modes of heat transfer.
For this situation, thermal images (which provide detailed temperature profiles of an object surface) may be used as a sensing system. 
We encourage interested readers to work along these open research directions.  

\section{APPENDIX}
\subsection{Online Estimation of the Interaction Matrix}
\label{section:discrete integral}
We use the numerical definite integral solver provided by $\textit{SciPy}$ \cite{2020SciPy-NMeth} to approximate the interaction matrix in real time. The essence of numerical integral is to divide the complex integrand into small subsections, and approximate each subsection with a polynomial that is easy to integrate. A commonly used method is the composite Simpson's rule\cite{hai2008elementary}, which approximates the subsection by quadratic polynomials. The general form of composite Simpson's rule is as follows:
\begin{equation}
\begin{split}
   &\int_{c}^{d} g(x) \diff x \approx \\
    & \frac{h}{3}\left[g\left(x_{0}\right)+2 \sum_{j=1}^{n / 2-1} g\left(x_{2 j}\right)+4 \sum_{j=1}^{n / 2} g\left(x_{2 j-1}\right)+g\left(x_{n}\right)\right]
\end{split}
\end{equation}
where $n$ is the number of subintervals, $x_j=c+jh$ for $j=\mathrm{0,1,...,n-1,n}$ with $h=(d-c)/n$; in particular, $x_{0}=c$ and $x_{n}=d$. 

\subsection{Evaluation of the Estimated Interaction Matrix}
\label{section:comparison}
In Section \ref{section:Arbitrary Surfaces at Arbitrary Configurations}, we propose a new method that uses truncated Fourier series to approximate the view factor between irregularly shaped objects. Since the interaction matrices are estimated according to the view factor, the accuracy and computation time of the approximation  directly affects the performance of the system. In this section, we conduct a case study where the standard discrete surface integral (DSI) method\cite{muneer2015finite} and our proposed method are used to calculate the view factor between the bunny object and the circular heat source used in experiments. The center of the heat source is $\mathbf{c}_{1}=\begin{bmatrix}
0&0&0
\end{bmatrix}^{\T}$ and the center of the bunny object is $\mathbf{c}_{2}=\begin{bmatrix}
0&0&5
\end{bmatrix}^{\T}$ (measured in \SI{}{\cm}). We denote the number of discrete surfaces by $n$ and the number of harmonics by $h$. 

Since there is no explicit formula to calculate the view factor of irregularly shaped objects, we take the estimated value $F_{21}=0.764602$ obtained by the DSI method when $n$ is set to $20000$ as the ground truth. The ground truth is then used to compute the error of the estimation when using different values of $n$ and $h$. The comparison result in terms of accuracy and computation time is depicted in Fig. \ref{fig: FEA_comparision}. We find that the result obtained by the proposed method converges to the ground truth in a short period of time. Thus, a desirable estimation of the interaction matrix can be obtained in real time.

\begin{figure}[t]
                \centering
                \includegraphics[width =1 \columnwidth]{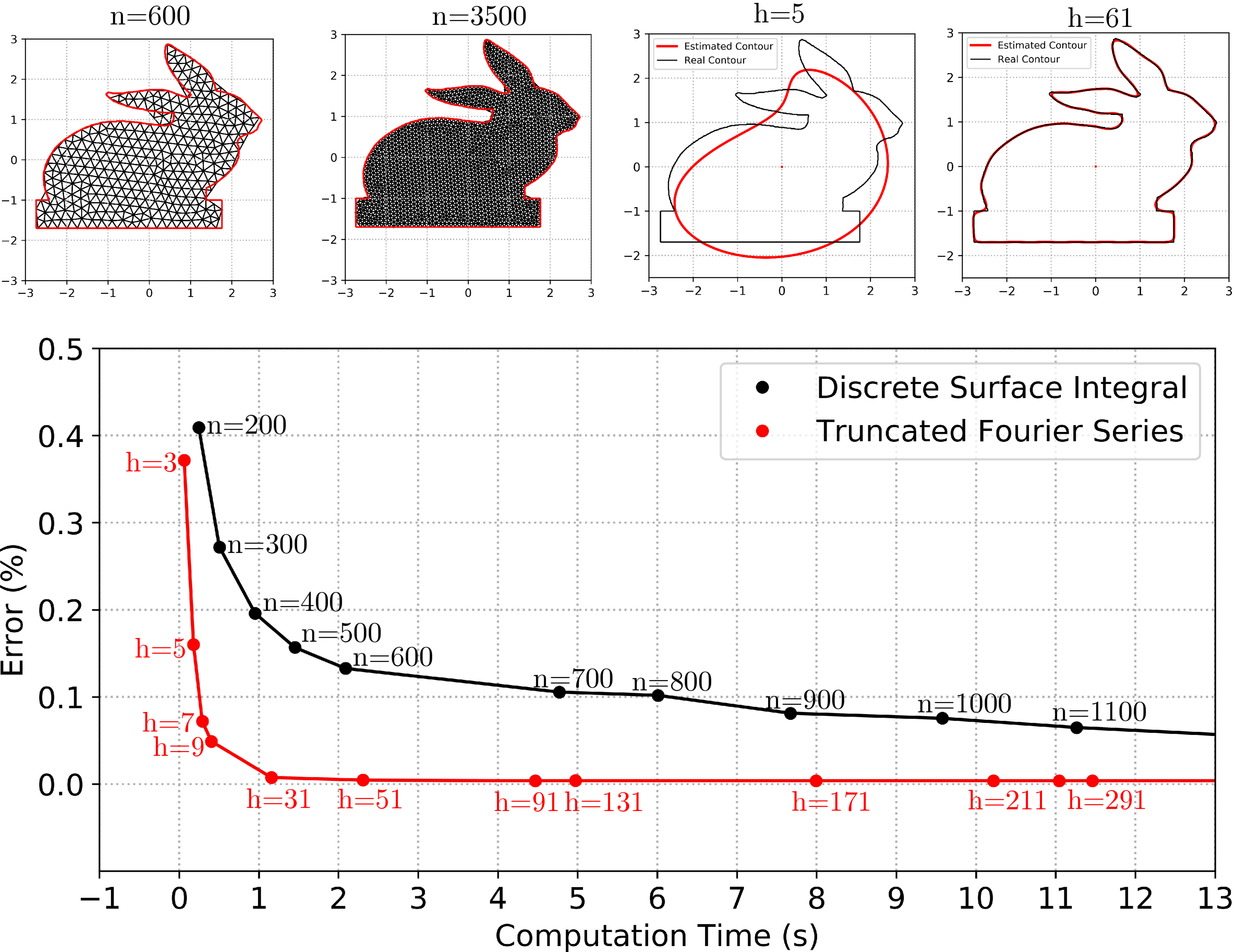}
                \caption{Comparison between the discrete surface integral method and the proposed truncated Fourier series algorithm. }
                \label{fig: FEA_comparision}
\end{figure}

\subsection{Polynomial Fitting with Sliding Window}
\label{sec:polynomial fitting}
Due to the limitation of the sensor accuracy and sampling rate, the raw feedback temperatures are discrete and noisy. This is undesirable for directly estimating the true values of $\mathbf{v}$. To this end, we implement polynomial fitting with a sliding window method, to estimate the values of $\boldsymbol{\tau}$ and $\mathbf{v}$ from the raw feedback data. For a single object, we denote its temperature as $T_{t}$ where $t$ is a variable of sampling time. A series of sampling times is denoted by $t_{1},t_{2},\cdots,t_{n}$. The size of the sliding window is set to be 10 data samples. When $n<10$, the estimation is at the initialization stage and the robot will not move. When $n\geq 10$, we denote the collected  10  temperature values by a vector $\mathbf{Y}_{s}=\begin{bmatrix}
T_{t_{n-9}}& T_{t_{n-8}}&\cdots&T_{t_{n}}
\end{bmatrix}^{\T}$. We then fit 10 sample points in the sliding window to the polynomial of order 3 as:
\begin{equation}
    \widehat{\mathbf{Y}}_{s}=\mathbf{P}\mathbf{c}^{t_{n}}=
    \begin{bmatrix}
    (t_{n-9})^{3}&(t_{n-9})^{2}&(t_{n-9})&1\\
    (t_{n-8})^{3}&(t_{n-8})^{2}&(t_{n-8})&1\\
    \vdots&\vdots&\vdots\\
    (t_{n})^{3}&(t_{n})^{2}&(t_{n})&1
    \end{bmatrix}
    \begin{bmatrix}
    c_{3}^{t_{n}}\\c_{2}^{t_{n}}\\c_{1}^{t_{n}}\\c_{0}^{t_{n}}\\
    \end{bmatrix}
\end{equation}
where $\mathbf{c}^{t_{n}}$ is the coefficient vector of the polynomial at sample time $t_{n}$. To minimize $\| \mathbf{Y}_{s}-\widehat{\mathbf{Y}}_{s} \|^2$, the coefficients are computed as $\mathbf{c}^{t_{n}}={\mathbf{P}}^{+}\mathbf{Y}_{s}$. Then, the temperature and temperature rate of a single object can be estimated as:
\begin{align}
    &T_{t_{n}}=c^{t_{n}}_{3}(t_{n})^3+c^{t_{n}}_{2}(t_{n})^2+c^{t_{n}}_{1}(t_{n})+c^{t_{n}}_{0}\\
    &v_{t_{n}}=3c^{t_{n}}_{3}(t_{n})^2+2c^{t_{n}}_{2}(t_{n})+c^{t_{n}}_{1}
\end{align}
When the new data is obtained, we update $\widehat{\mathbf{Y}}_{s}$ and $\mathbf{P}$. By following the same procedure, object temperature rate can be estimated accordingly.

\iffalse
\subsection{\comments{Calibration}}
\label{auto-calirbarion}
We denote the position of a 3D point obtained by the depth camera as $\mathbf{\Psi}\in \mathbb{R}^{3 \times 3}$. After conducting calibration between the depth camera and the thermal camera, we obtain the intrinsic and extrinsic thermal camera matrix $\mathbf{K}\in \mathbb{R}^{3 \times 3}, \mathbf{E} \in \mathbb{R}^{3 \times 4}$. According to the camera perspective projection relation, a 2D image point $\mathbf{\psi}$ on the thermal image is related to a 3D point $\mathbf{\Psi}$ as follows.
\begin{equation}
    \begin{bmatrix}
    \boldsymbol{\psi}\\1
    \end{bmatrix}=\mathbf{K}\mathbf{E}\begin{bmatrix}
    \mathbf{\Psi}\\1
    \end{bmatrix}
\end{equation}
Using the above equation, we can find the corresponding temperature of each 3D point. 3D points that have high temperature are classified as the surface of the heat source, and the corresponding surface normal $\vec{\mathbf{k}}_{1}$, center $\mathbf{c}_{1}$ can be determined. We then utilize standard hand-eye calibration to find the transformation matrix between the robot end-effector and the depth camera. Finally, we obtain the geometric relationship between the robot, the heat source, and the end-effector, the system is calibrated.
\fi
\ifCLASSOPTIONcaptionsoff
  \newpage
\fi

\bibliography{bibliography.bib}

\begin{thebibliography}{10}
\providecommand{\url}[1]{#1}
\csname url@rmstyle\endcsname
\providecommand{\newblock}{\relax}
\providecommand{\bibinfo}[2]{#2}
\providecommand\BIBentrySTDinterwordspacing{\spaceskip=0pt\relax}
\providecommand\BIBentryALTinterwordstretchfactor{4}
\providecommand\BIBentryALTinterwordspacing{\spaceskip=\fontdimen2\font plus
\BIBentryALTinterwordstretchfactor\fontdimen3\font minus
  \fontdimen4\font\relax}
\providecommand\BIBforeignlanguage[2]{{%
\expandafter\ifx\csname l@#1\endcsname\relax
\typeout{** WARNING: IEEEtran.bst: No hyphenation pattern has been}%
\typeout{** loaded for the language `#1'. Using the pattern for}%
\typeout{** the default language instead.}%
\else
\language=\csname l@#1\endcsname
\fi
#2}}

\bibitem{dna2014_ijrr}
D.~Navarro-Alarcon, Y.-H. Liu, J.~G. Romero, and P.~Li, ``On the visual
  deformation servoing of compliant objects: Uncalibrated control methods and
  experiments,'' \emph{{Int. Journal of Robotics Research}}, vol.~33, no.~11,
  pp. 1462--1480, 2014.

\bibitem{Journals:Cherubini2013}
A.~Cherubini and F.~Chaumette, ``Visual navigation of a mobile robot with
  laser-based collision avoidance,'' \emph{{Int. Journal of Robotics
  Research}}, vol.~32, no.~2, pp. 189--205, 2013.

\bibitem{marie2020_ral}
M.~Tirindelli, M.~Victorova, J.~Esteban, S.~T. Kim, D.~Navarro-Alarcon,
  \emph{et~al.}, ``Force-ultrasound fusion: Bringing spine robotic-us to the
  next ``level'','' \emph{{IEEE Robot. Autom. Lett.}}, vol.~5, no.~4, pp.
  5661--5668, 2020.

\bibitem{Proceedings:Magassouba2016}
A.~Magassouba, N.~Bertin, and F.~Chaumette, ``Audio-based robot control from
  interchannel level difference and absolute sound energy,'' in \emph{Proc.
  {IEEE} Int. Conf. Intelligent Robots and Systems}, 2016, pp. 1992--1999.

\bibitem{Proceedings:Rahbar2017}
F.~{Rahbar}, A.~{Marjovi}, P.~{Kibleur}, and A.~{Martinoli}, ``A 3-d
  bio-inspired odor source localization and its validation in realistic
  environmental conditions,'' in \emph{{IEEE/RSJ Int. Conf. on Intelligent
  Robots and Systems}}, 2017, pp. 3983--3989.

\bibitem{Journals:Gade2014}
R.~Gade and T.~B. Moeslund, ``Thermal cameras and applications: A survey,''
  \emph{Machine Vision and Applications}, vol.~25, pp. 245--262, 2014.

\bibitem{Journals:Benli2019}
E.~{Benli}, R.~L. {Spidalieri}, and Y.~{Motai}, ``Thermal multisensor fusion
  for collaborative robotics,'' \emph{IEEE Transactions on Industrial
  Informatics}, vol.~15, no.~7, pp. 3784--3795, 2019.

\bibitem{Journals:Cao2018}
Y.~Cao, B.~Xu, Z.~Ye, J.~Yang, Y.~Cao, \emph{et~al.}, ``Depth and thermal
  sensor fusion to enhance 3d thermographic reconstruction,'' \emph{Opt.
  Express}, vol.~26, no.~7, pp. 8179--8193, Apr 2018.

\bibitem{Jorunals:He2015}
Y.~{He} and R.~{Yang}, ``Eddy current volume heating thermography and phase
  analysis for imaging characterization of interface delamination in cfrp,''
  \emph{IEEE Trans. on Industrial Informatics}, vol.~11, no.~6, pp. 1287--1297,
  2015.

\bibitem{Journals:Lai2015}
W.~W.-L. Lai, K.-K. Lee, and C.-S. Poon, ``Validation of size estimation of
  debonds in external wall’s composite finishes via passive infrared
  thermography and a gradient algorithm,'' \emph{Construction and Building
  Materials}, vol.~87, pp. 113 -- 124, 2015.

\bibitem{cherubini2020_frontneuro}
A.~Cherubini and D.~Navarro-Alarcon, ``Sensor-based control for human-robot
  collaboration: Fundamentals, challenges and opportunities,'' \emph{Front. in
  Neurorobotics (in press)}, vol.~1, no.~1, pp. 1--21, 2020.

\bibitem{fuel_cell}
W.~Binrui, J.~Yinglian, X.~Hong, and W.~Ling, ``Temperature control of pem fuel
  cell stack application on robot using fuzzy incremental pid,'' in \emph{2009
  Chinese Control and Decision Conference}.\hskip 1em plus 0.5em minus
  0.4em\relax IEEE, 2009, pp. 3293--3297.

\bibitem{shape_memory_alloy}
M.~Ho and J.~P. Desai, ``Towards a mri-compatible meso-scale sma-actuated robot
  using pwm control,'' in \emph{2010 3rd IEEE RAS \& EMBS International
  Conference on Biomedical Robotics and Biomechatronics}.\hskip 1em plus 0.5em
  minus 0.4em\relax IEEE, 2010, pp. 361--366.

\bibitem{welding}
J.~De~Backer, G.~Bolmsj{\"o}, and A.-K. Christiansson, ``Temperature control of
  robotic friction stir welding using the thermoelectric effect,'' \emph{The
  International Journal of Advanced Manufacturing Technology}, vol.~70, no.
  1-4, pp. 375--383, 2014.

\bibitem{firefighting}
J.-H. Kim and B.~Y. Lattimer, ``Real-time probabilistic classification of fire
  and smoke using thermal imagery for intelligent firefighting robot,''
  \emph{Fire Safety Journal}, vol.~72, pp. 40--49, 2015.

\bibitem{solar_tracking}
A.~Cammarata, ``Optimized design of a large-workspace 2-dof parallel robot for
  solar tracking systems,'' \emph{Mechanism and machine theory}, vol.~83, pp.
  175--186, 2015.

\bibitem{muddassir2020_tmech}
M.~Muddassir, D.~Gomez, L.~Hu, S.~Chen, and D.~Navarro-Alarcon, ``Robotics
  meets cosmetic dermatology: Development of a novel vision-guided system for
  skin photo-rejuvenation,'' \emph{{IEEE/ASME} Trans. Mechatronics (under
  review)}, vol.~1, no.~1, pp. 1--15, 2020.

\bibitem{Journals:McKemy2007}
D.~{McKemy}, ``Temperature sensing across species,'' \emph{{Eur. J. Physiol.}},
  vol. 454, pp. 777--791, 2007.

\bibitem{car_mold}
J.~Mlynek, R.~Knobloch, and R.~Srb, ``Optimization of a heat radiation
  intensity and temperature field on the mould surface.'' in \emph{ECMS}, 2016,
  pp. 425--431.

\bibitem{Proceedings:Imdoukh2017}
A.~{Imdoukh}, A.~{Shaker}, A.~{Al-Toukhy}, D.~{Kablaoui}, and M.~{El-Abd},
  ``Semi-autonomous indoor firefighting uav,'' in \emph{Int. Conf. on Advanced
  Robotics}, 2017, pp. 310--315.

\bibitem{volcano}
G.~Muscato, F.~Bonaccorso, L.~Cantelli, D.~Longo, and C.~D. Melita, ``Volcanic
  environments: Robots for exploration and measurement,'' \emph{IEEE Robotics
  \& Automation Magazine}, vol.~19, no.~1, pp. 40--49, 2012.

\bibitem{Journals:Fu2016}
M.~Fu, W.~Weng, W.~Chen, and N.~Luo, ``Review on modeling heat transfer and
  thermoregulatory responses in human body,'' \emph{Journal of Thermal
  Biology}, vol.~62, pp. 189 -- 200, 2016, modeling bioheat transfer processes
  and thermoregulatory responses.

\bibitem{Journals:Lee2003}
R.~Lee, ``Thermal tolerances of deep-sea hydrothermal vent animals from the
  northeast pacific,'' \emph{Biological Bulletin}, vol. 205, no.~2, pp.
  98--101, 2003.

\bibitem{butterflies}
C.-C. Tsai, R.~A. Childers, N.~N. Shi, C.~Ren, J.~N. Pelaez, \emph{et~al.},
  ``Physical and behavioral adaptations to prevent overheating of the living
  wings of butterflies,'' \emph{Nature communications}, vol.~11, no.~1, pp.
  1--14, 2020.

\bibitem{insect_thermoregulationt}
M.~L. May, ``Insect thermoregulation,'' \emph{Annual review of entomology},
  vol.~24, no.~1, pp. 313--349, 1979.

\bibitem{Journals:Natale2002}
L.~Natale, G.~Metta, and G.~Sandini, ``Development of auditory-evoked reflexes:
  Visuo-acoustic cues integration in a binocular head,'' \emph{{Rob. Auton.
  Syst.}}, vol.~39, no.~2, pp. 87 -- 106, 2002.

\bibitem{Journals:Arechavaleta2008}
G.~{Arechavaleta}, J.~{Laumond}, H.~{Hicheur}, and A.~{Berthoz}, ``An
  optimality principle governing human walking,'' \emph{{IEEE Trans. Robot.}},
  vol.~24, no.~1, pp. 5--14, 2008.

\bibitem{Journals:Na2020}
S.~Na, Y.~Qiu, A.~E. Turgut, J.~Ulrich, T.~Krajník, \emph{et~al.},
  ``Bio-inspired artificial pheromone system for swarm robotics applications,''
  \emph{Adaptive Behavior}, pp. 1--21, 2020.

\bibitem{Fundamentals_of_heat_and_mass_transfer}
T.~L. Bergman, F.~P. Incropera, D.~P. DeWitt, and A.~S. Lavine,
  \emph{Fundamentals of heat and mass transfer}.\hskip 1em plus 0.5em minus
  0.4em\relax John Wiley \& Sons, 2011.

\bibitem{modest2013radiative}
M.~F. Modest, \emph{Radiative heat transfer}.\hskip 1em plus 0.5em minus
  0.4em\relax Academic press, 2013.

\bibitem{lienhard2005heat}
I.~Lienhard and H.~John, \emph{A heat transfer textbook}.\hskip 1em plus 0.5em
  minus 0.4em\relax phlogiston press, 2005.

\bibitem{viewfactor1}
M.~Vuji{\v{c}}i{\'c}, N.~Lavery, and S.~Brown, ``View factor calculation using
  the monte carlo method and numerical sensitivity,'' \emph{Communications in
  numerical methods in Engineering}, vol.~22, no.~3, pp. 197--203, 2006.

\bibitem{viewfactor2}
J.~C. Chai, J.~P. Moder, and K.~C. Karki, ``A procedure for view factor
  calculation using the finite-volume method,'' \emph{Numerical Heat Transfer:
  Part B: Fundamentals}, vol.~40, no.~1, pp. 23--35, 2001.

\bibitem{viewfactor3}
S.~C. Mishra, A.~Shukla, and V.~Yadav, ``View factor calculation in the 2-d
  geometries using the collapsed dimension method,'' \emph{International
  communications in heat and mass transfer}, vol.~35, no.~5, pp. 630--636,
  2008.

\bibitem{vujivcic2006numerical}
M.~Vuji{\v{c}}i{\'c}, N.~Lavery, and S.~Brown, ``Numerical sensitivity and view
  factor calculation using the monte carlo method,'' \emph{Proceedings of the
  Institution of Mechanical Engineers, Part C: Journal of Mechanical
  Engineering Science}, vol. 220, no.~5, pp. 697--702, 2006.

\bibitem{sparrow2018radiation}
E.~M. Sparrow, \emph{Radiation heat transfer}.\hskip 1em plus 0.5em minus
  0.4em\relax Routledge, 2018.

\bibitem{rao1996efficient}
V.~R. Rao and V.~Sastri, ``Efficient evaluation of diffuse view factors for
  radiation,'' \emph{International journal of heat and mass transfer}, vol.~39,
  no.~6, pp. 1281--1286, 1996.

\bibitem{flanders1973differentiation}
H.~Flanders, ``Differentiation under the integral sign,'' \emph{The American
  Mathematical Monthly}, vol.~80, no.~6, pp. 615--627, 1973.

\bibitem{2020SciPy-NMeth}
P.~Virtanen, R.~Gommers, T.~E. Oliphant, M.~Haberland, T.~Reddy, \emph{et~al.},
  ``{{SciPy} 1.0: Fundamental Algorithms for Scientific Computing in Python},''
  \emph{Nature Methods}, vol.~17, pp. 261--272, 2020.

\bibitem{tolstov2012fourier}
G.~P. Tolstov, \emph{Fourier series}.\hskip 1em plus 0.5em minus 0.4em\relax
  Courier Corporation, 2012.

\bibitem{navarro2018fourier}
D.~Navarro-Alarcon and Y.-H. Liu, ``Fourier-based shape servoing: a new
  feedback method to actively deform soft objects into desired 2-d image
  contours,'' \emph{IEEE Transactions on Robotics}, vol.~34, no.~1, pp.
  272--279, 2018.

\bibitem{muneer2015finite}
T.~Muneer, S.~Ivanova, Y.~Kotak, and M.~Gul, ``Finite-element view-factor
  computations for radiant energy exchanges,'' \emph{Journal of Renewable and
  Sustainable Energy}, vol.~7, no.~3, p. 033108, 2015.

\bibitem{Book:Nakamura1991}
Y.~Nakamura, \emph{Advanced robotics: redundancy and optimization}.\hskip 1em
  plus 0.5em minus 0.4em\relax Boston, MA: Addison-Wesley Longman, 1991.

\bibitem{vidyasagar2002nonlinear}
M.~Vidyasagar, \emph{Nonlinear systems analysis}.\hskip 1em plus 0.5em minus
  0.4em\relax SIAM, 2002.

\bibitem{titchener2015calculation}
N.~Titchener, S.~Colliss, and H.~Babinsky, ``On the calculation of
  boundary-layer parameters from discrete data,'' \emph{Experiments in Fluids},
  vol.~56, no.~8, p. 159, 2015.

\bibitem{Journals:Liu2006}
Y.-H. Liu, H.~Wang, C.~Wang, and K.~K. Lam, ``Uncalibrated visual servoing of
  robots using a depth-independent interaction matrix,'' \emph{{IEEE} Trans.
  Robot.}, vol.~22, no.~4, pp. 804--817, Aug. 2006.

\bibitem{garrido2014automatic}
S.~Garrido-Jurado, R.~Mu{\~n}oz-Salinas, F.~J. Madrid-Cuevas, and M.~J.
  Mar{\'\i}n-Jim{\'e}nez, ``Automatic generation and detection of highly
  reliable fiducial markers under occlusion,'' \emph{Pattern Recognition},
  vol.~47, no.~6, pp. 2280--2292, 2014.

\bibitem{dna2020_front_neuro}
D.~Navarro-Alarcon, J.~Qi, J.~Zhu, and A.~Cherubini, ``A {Lyapunov}-stable
  adaptive method to approximate sensorimotor models for sensor-based
  control,'' \emph{{Front. in Neurorobotics}}, vol.~14, no.~59, pp. 1--12,
  2020.

\bibitem{williams2007research}
C.~Williams, ``Research methods,'' \emph{Journal of Business \& Economics
  Research (JBER)}, vol.~5, no.~3, 2007.

\bibitem{hai2008elementary}
D.~Hai and R.~Smith, ``An elementary proof of the error estimates in simpson's
  rule,'' \emph{Mathematics Magazine}, vol.~81, no.~4, pp. 295--300, 2008.

\end{thebibliography}
\bibliographystyle{IEEEtran}

\begin{IEEEbiography} [{\includegraphics[width=1in,height=1.25in,clip,keepaspectratio]{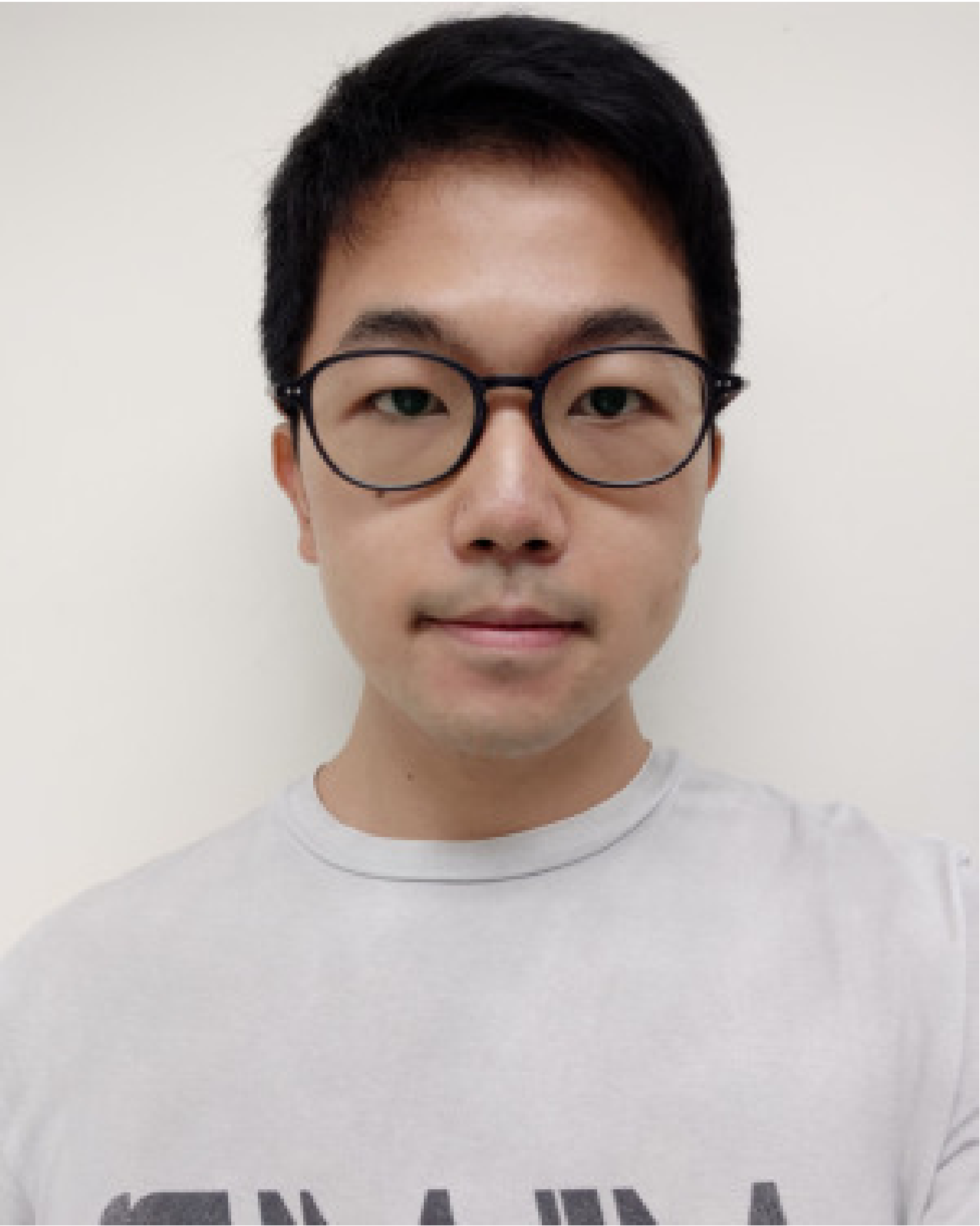}}] 
	{Luyin Hu} received his BEng. degree in Mechanical Engineering from The Hong Kong Polytechnic University, KLN, Hong Kong in 2019. Currently, he is pursuing an MPhil degree in Mechanical Engineering at the same university, where he also works as a Research Assistant.
	His research interests include multimodal robot perception, servomechanisms, and control system design.
\end{IEEEbiography}

\begin{IEEEbiography} [{\includegraphics[width=1in,height=1.25in,clip,keepaspectratio]{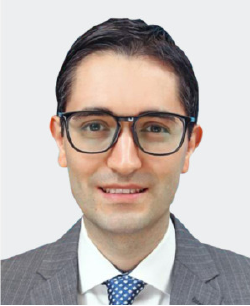}}]
	{David Navarro-Alarcon} (GS'06--M'14--SM'19) received the Ph.D. degree in Mechanical and Automation Engineering from The Chinese University of Hong Kong (CUHK), Shatin, Hong Kong, in 2014.
	
	He was an Assistant (Research) Professor at the CUHK T Stone Robotics Institute, from 2015 to 2017. 
	Since 2017, he has been with The Hong Kong Polytechnic University, KLN, Hong Kong, where he is an Assistant Professor at the Department of Mechanical Engineering. 
	His current research interests include perceptual robotics and control theory.
\end{IEEEbiography}

\begin{IEEEbiography}
[{\includegraphics[width=1in,height=1.25in,clip,keepaspectratio]{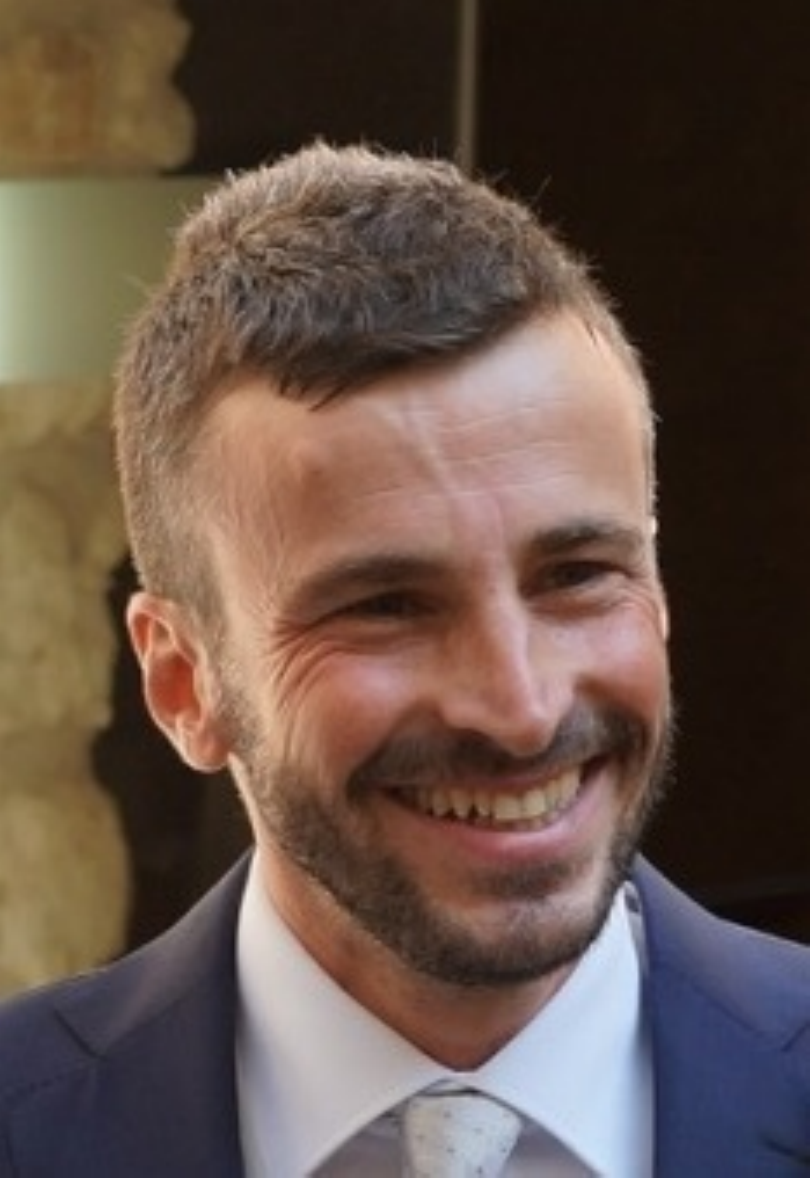}}]
{Andrea Cherubini} received the M.Sc. in Mechanical Engineering in 2001 from the University of Rome La Sapienza and a second M.Sc. in Control Systems in 2003 from the University of Sheffield, U.K. In 2008, he received the Ph.D. in Control Systems from La Sapienza. From 2008 to 2011, he was postdoc at INRIA Rennes. Since 2011, he is Associate Professor at Universit\'e de Montpellier. His research interests include: physical human–robot interaction, and manipulation of soft objects.
\end{IEEEbiography}

\begin{IEEEbiography}
[{\includegraphics[width=1in,height=1.25in,clip,keepaspectratio]{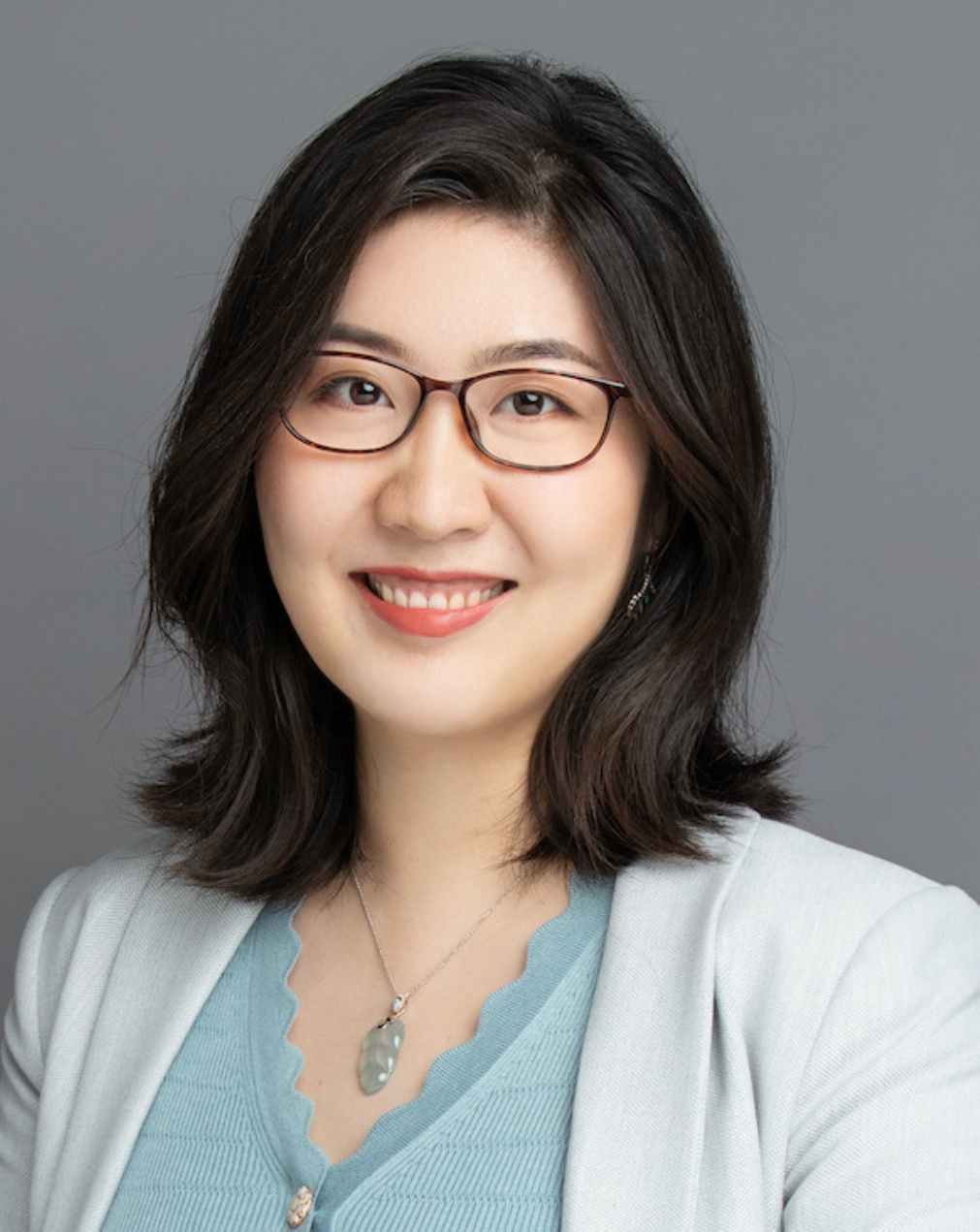}}]
{Mengying Li} received her PhD degree in Mechanical and Aerospace Engineering from University of California San Diego. From 2018 to 2020, She was a Postdoctoral Scholar in the Center for Energy Research of UC San Diego. She joined The Hong Kong Polytechnic University as an Assistant Professor in the Department of Mechanical Engineering since 2020. Her research interests include atmospheric radiative heat transfer for renewable energy integration and design of multi-generation systems.
\end{IEEEbiography}

\end{document}